\documentclass{article}
\pdfoutput=1
\usepackage{arxiv}
\setcitestyle{authoryear,open={(},close={)}}
\renewcommand{\cite}{\citep}

\usepackage[utf8]{inputenc} 
\usepackage[T1]{fontenc}    
\usepackage{hyperref}
\usepackage{url}            
\usepackage{booktabs}       
\usepackage{amsfonts}       
\usepackage{nicefrac}       
\usepackage{microtype}      
\usepackage{graphicx}
\usepackage{amsmath}
\usepackage{amsthm}
\usepackage{cases}
\usepackage{listings}
\usepackage{color}
\usepackage{textcomp}
\usepackage{subfig}
\usepackage{parskip}
\usepackage{xcolor}
\usepackage{tikz}
\usepackage{multirow}

\def\shownotes{1}
\ifnum\shownotes=1
\newcommand{\authnote}[2]{[#1: #2]}
\else
\newcommand{\authnote}[2]{}
\fi

\newcommand{\spoc}{\textsc{SPoC}}
\newcommand{\sanstype}{\textsc{SansType}}
\newcommand{\testp}{\textsc{TestP}}
\newcommand{\testw}{\textsc{TestW}}

\newcommand{\hatalpha}{\hat{\alpha}}
\newcommand{\fstd}{f_{\text{std}}}
\newcommand{\fstdi}{f_{\text{std},j}}

\newcommand{\ftheta}{f_{\theta}}
\newcommand{\fthetai}{f_{\theta_j}}

\newcommand{\fthetabasemin}{f_{\text{base}}}
\newcommand{\fthetabase}{f_{\text{base}}}

\newcommand{\fhat}{\hat{f}}
\newcommand{\fstar}{f^\star}

\DeclareMathOperator*{\argmin}{arg\,min}
\DeclareMathOperator*{\argmax}{arg\,max}

\newlength{\widebarargwidth}
\newlength{\widebarargheight}
\newlength{\widebarargdepth}

\makeatletter
\newenvironment{btHighlight}[1][]
{\begingroup\tikzset{bt@Highlight@par/.style={#1}}\begin{lrbox}{\@tempboxa}}
{\end{lrbox}\bt@HL@box[bt@Highlight@par]{\@tempboxa}\endgroup}

\newcommand\btHL[1][]{%
  \begin{btHighlight}[#1]\bgroup\aftergroup\bt@HL@endenv%
}
\def\bt@HL@endenv{%
  \end{btHighlight}%
  \egroup
}
\newcommand{\bt@HL@box}[2][]{%
  \tikz[#1]{%
    \pgfpathrectangle{\pgfpoint{1pt}{0pt}}{\pgfpoint{\wd #2}{\ht #2}}%
    \pgfusepath{use as bounding box}%
    \node[anchor=base west, fill=orange!30,outer sep=0pt,inner xsep=1pt, inner ysep=0pt, rounded corners=3pt, minimum height=\ht\strutbox+1pt,#1]{\raisebox{1pt}{\strut}\strut\usebox{#2}};
  }%
}
\let\orig@lstnumber=\thelstnumber

\newcommand\lstsetnumber[1]{\gdef\thelstnumber{#1}}
\newcommand\lstresetnumber{\global\let\thelstnumber=\orig@lstnumber}
\makeatother
\definecolor{listinggray}{gray}{0.9}
\definecolor{lbcolor}{rgb}{0.9,0.9,0.9}
\newcommand\sH{\ensuremath{\mathcal{H}}}

\newcommand\sR{\ensuremath{\mathcal{R}}}

\newcommand\sV{\ensuremath{\mathcal{V}}}

\newcommand\sX{\ensuremath{\mathcal{X}}}
\newcommand\sY{\ensuremath{\mathcal{Y}}}

\newcommand\R{\ensuremath{\mathbb{R}}} 

\ifthenelse{\isundefined{\definition}}{\newtheorem{definition}{Definition}}{}
\ifthenelse{\isundefined{\assumption}}{\newtheorem{assumption}{Assumption}}{}
\ifthenelse{\isundefined{\hypothesis}}{\newtheorem{hypothesis}{Hypothesis}}{}
\ifthenelse{\isundefined{\proposition}}{\newtheorem{proposition}{Proposition}}{}
\ifthenelse{\isundefined{\theorem}}{\newtheorem{theorem}{Theorem}}{}
\ifthenelse{\isundefined{\lemma}}{\newtheorem{lemma}{Lemma}}{}
\ifthenelse{\isundefined{\corollary}}{\newtheorem{corollary}{Corollary}}{}
\ifthenelse{\isundefined{\alg}}{\newtheorem{alg}{Algorithm}}{}
\ifthenelse{\isundefined{\example}}{\newtheorem{example}{Example}}{}
\newcommand{\E}{\ensuremath{\mathbb{E}}} 

\begin{document}
\bibliographystyle{plainnat}

\title{Composed Fine-Tuning: Freezing Pre-Trained Denoising Autoencoders for Improved Generalization}

\author{%
        \large{Sang Michael Xie} \\
        \large{Stanford University}\\
        {\texttt{xie@cs.stanford.edu}} \\
        \And
        \large{Percy Liang} \\
        \large{Stanford University} \\
        {\texttt{pliang@cs.stanford.edu}} \\
        \And
        \large{Tengyu Ma} \\
        \large{Stanford University} \\
        {\texttt{tengyuma@cs.stanford.edu}} \\
}

\date{}

\newcommand{\fix}{\marginpar{FIX}}
\newcommand{\new}{\marginpar{NEW}}

\maketitle

\begin{abstract}
    We focus on prediction problems with structured outputs that are subject to output validity constraints, e.g. pseudocode-to-code translation where the code must compile. While labeled input-output pairs are expensive to obtain, ``unlabeled'' outputs, i.e. outputs without corresponding inputs, are freely available (e.g. code on GitHub) and provide information about output validity.  We can capture the output structure by pre-training a denoiser to denoise corrupted versions of unlabeled outputs. We first show that standard fine-tuning after pre-training destroys some of this structure. We then propose \emph{composed fine-tuning}, which fine-tunes a predictor composed with the pre-trained denoiser, which is frozen to preserve output structure.  For two-layer ReLU networks, we prove that composed fine-tuning significantly reduces the complexity of the predictor, thus improving generalization.  Empirically, we show that composed fine-tuning improves over standard fine-tuning on two pseudocode-to-code translation datasets (3\% and 6\% relative). The improvement from composed fine-tuning is magnified on out-of-distribution (OOD) examples (4\% and 25\% relative).
\end{abstract}

\section{Introduction}

We study prediction problems whose outputs have validity constraints. For example, in pseudocode-to-code translation~\cite{kulal2019spoc}, the output code must compile. Other examples include machine translation~\cite{sutskever2014sequence,cho2014statmt,bahdanau2015neural} and molecule generation~\cite{lucio2020molecule,senior2020protein}, where outputs should be grammatical or chemically valid, respectively.
Expensively-obtained labeled data may not contain enough examples to learn a model that captures the complex output structure governing validity.
In these problems, there are often lots of ``unlabeled'' outputs---outputs without a corresponding input (e.g., GitHub has over 40 million public code repositories as of February 2021~\cite{github2021repos}).

\begin{figure}[tbp]
  \centering
  \subfloat{\includegraphics[width=0.80\textwidth]{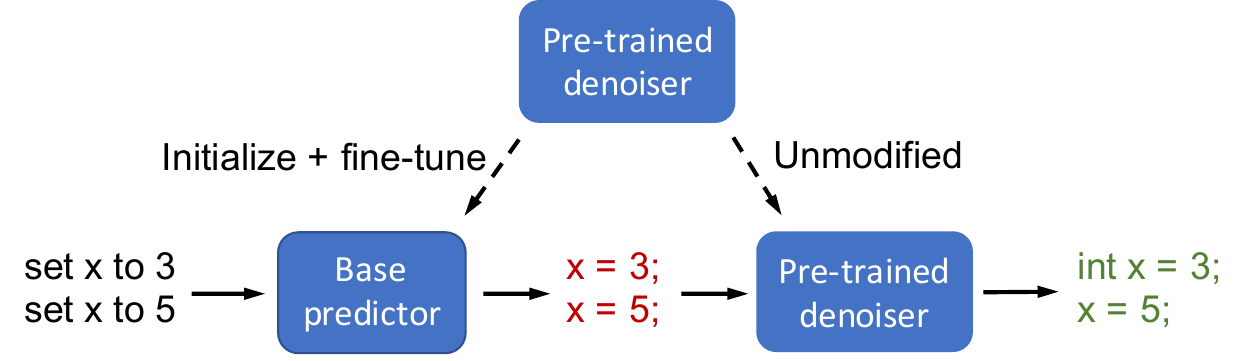}}
  \hfill
  \subfloat{\includegraphics[width=0.35\textwidth]{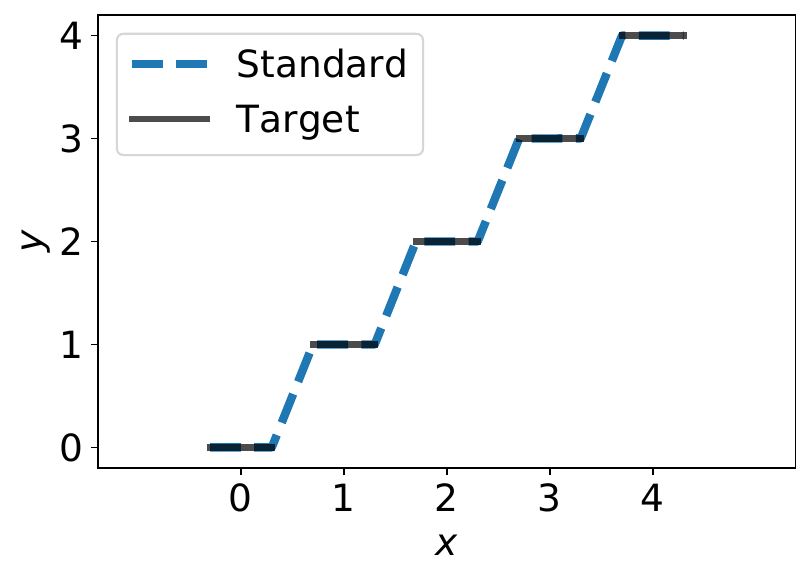}}
  \subfloat{\includegraphics[width=0.35\textwidth]{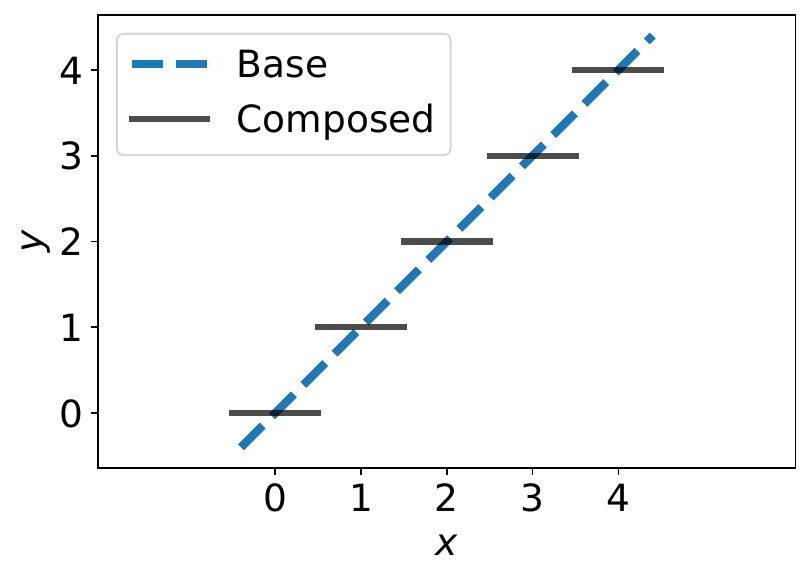}}
  \caption{\textbf{(Top)} Composed fine-tuning: we fine-tune a base predictor composed with a pre-trained denoising autoencoder (denoiser) that is unmodified during fine-tuning. The base predictor is initialized from the same pre-trained denoiser. Composing with the fixed denoiser preserves the output structure information in the pre-trained denoiser. This offloads the complexity of capturing output structure to the denoiser, reducing the complexity of the base predictor. \textbf{(Bottom-Left)} Univariate regression example where the target function is a staircase defined on disjoint intervals. Standard fine-tuning fits the target directly, which requires learning a complex function. \textbf{(Bottom-Right)} A simple linear function can be used to fit a staircase function when composed with the denoiser that projects onto the valid outputs (the integers).
  }\label{fig:simple}
  \end{figure}

Pre-training with denoising autoencoders (e.g., BART~\citep{lewis2020bart} and T5~\citep{raffel2019exploring}) leverages unlabeled outputs to learn the output structure and transfer it to downstream tasks.
The idea is to synthetically corrupt unlabeled outputs and learn to denoise them, yielding a pre-trained denoiser that generates valid outputs.
The pre-trained denoiser can be fine-tuned on a smaller labeled dataset from a downstream task such as machine translation.
Fine-tuning is the de facto paradigm in NLP~\citep{devlin2019bert,lewis2020bart,raffel2019exploring} and vision~\citep{dosovitskiy2021vit,radford2021clip,kornblith2019better,yosinski2014transferable}, and is popular because the computationally expensive step of leveraging large-scale data is concentrated in pre-training, while fine-tuning on downstream tasks is relatively cheap.

In this paper, we first show that standard fine-tuning may not optimally leverage the pre-trained model in denoising autoencoder language models (e.g., BART, T5), complementing recent works on the suboptimality of fine-tuning in masked language models~\cite{dodge2020finetuning,lee2020mixout,zhang2021revisiting}.
We run a simple experiment where we apply the pre-trained denoiser post-hoc to the outputs of the standard fine-tuned model, improving the accuracy of the model by 0.5--1.5\% on a pseudocode-to-code dataset (Section~\ref{sec:code}).
Thus, it seems that standard fine-tuning intuitively ``destroys'' some of the information in the pre-trained weights.

To better utilize the pre-trained denoiser, we propose \emph{composed fine-tuning} (Figure~\ref{fig:simple} top), which fine-tunes a base predictor composed with a fixed pre-trained denoiser.
Here, the denoiser aims to enforce the output validity structure learned during pre-training.
The main difference with our simple experiment is that we fine-tune with the composed denoiser instead of only applying the denoiser at test-time.

Factorizing the prediction problem into two modules, the base predictor and the denoiser, allows the base predictor to be simpler by offloading the complexity of modeling output structure to the denoiser.
For intuition, Figure 1 (bottom-left/right) shows a pictorial example of a staircase function where valid outputs are integers and requires a complex spline to represent directly with standard fine-tuning.
In contrast, composing a simple linear base predictor with a denoiser (which rounds to the nearest integer) can not only represent the staircase function but also extrapolate perfectly outside the range of training data.
Thus, composed fine-tuning reduces the complexity, and thus sample efficiency, of the learned base predictor.

We formalize this intuition theoretically showing that composing with a denoiser reduces the complexity of the base predictor needed to represent the target function. We show this complexity gap can be arbitrarily large for 2-layer ReLU networks representing a family of functions from $\R\rightarrow \R^k$, depending on the complexity of the target function to be represented.

Empirically, on two pseudocode-to-code datasets, we show that composed fine-tuning improves validity and correctness over standard fine-tuning.
We consider the more difficult full-program translation task rather than line-by-line translation (with compiler side information) studied by previous work~\cite{kulal2019spoc,yasunaga2020repair}.
First we introduce \sanstype, a synthetic pseudocode-to-code dataset where the pseudocode specifies all but the variable types.
On \sanstype, composed fine-tuning corrects invalid code at test time and also simplifies the base predictor by helping with global type inference.
Secondly on \spoc~\cite{kulal2019spoc}, a recent pseudocode-to-code dataset based on programming competition problems, we improve the proportion of correct programs by about 1\% over standard fine-tuning and 3--4\% over a baseline Transformer.
We also show that composed fine-tuning provides gains on top of other semi-supervised methods, including backtranslation, which uses a large unlabeled dataset during fine-tuning. This is a more computationally expensive way to use unlabeled data than pre-training, which processes unlabeled data separately from downstream labeled data. Combining composed fine-tuning with backtranslation obtains further gains over backtranslation by about 1.5\% on \spoc.

Composed fine-tuning provides larger gains on out-of-distribution (OOD) examples, including on unseen pseudocode in \sanstype~and unseen programs in \spoc.
On \sanstype, the relative improvement over standard fine-tuning grows from 5\% to 24\% on the OOD split. On \spoc, the relative improvement increases from 3\% to 4\% on unseen programs.
We also test on an image generation task where all test inputs are OOD. In particular, the inputs are font type and character type pairs, and the output is a valid font character image. At test time, the composed predictor produces font images of unseen font-character pairings with noticeably better clarity and styling than a baseline that directly generates the image.

\section{Setup}
\label{sec:setup}
We consider prediction problems from an input space $\sX$ (e.g., pseudocode) to
an output space $\sY$ (e.g., code) where there is an unknown subset of
\emph{valid} outputs $\sV \subseteq \sY$ (e.g., code that compiles) and the
true output is always valid (in $\sV$).  A pre-trained \emph{denoiser}
$\Pi:\sY \rightarrow \sV$ ``projects'' a possibly invalid output in $\sY$ to
the valid set $\sV$.  In our theory, the denoiser is a Euclidean projection
onto the valid set $\sV$, and in practice, the denoiser is approximate and
learned on synthetic corrputions of valid unlabeled outputs.  Our goal is to is
to learn a \emph{predictor} $f: \sX \rightarrow \sY$ using a small downstream
labeled dataset $(x_1, y_1), \dots, (x_n, y_n)$ where $x_i\in\sX$ and $y_i\in \sV$.

\paragraph{Base and composed predictors.}
Let $\Pi \circ \fthetabase$ be a \emph{composed predictor} that approximates the target function $\fstar$.
In the context of $\Pi \circ \fthetabase$, we call $\fthetabase$ the \emph{base predictor}.

\paragraph{Standard fine-tuning.}
Standard fine-tuning parameterizes a probabilistic model $p_\alpha$ for the predictor and optimizes the parameters $\alpha$ by maximizing the log-likelihood on labeled data $\E_{x,y}[ \log p_\alpha( y \mid x)]$.
The parameters $\alpha$ are initialized with the parameters of the pre-trained denoiser.
We define the standard fine-tuned predictor to be $\fstd(x) = \argmax_y p_\alpha(y \mid x)$.
We also compare to other fine-tuning methods that use extra unlabeled data or initially freeze some layers in Section~\ref{sec:ablations}.

\newsavebox{\boxonea}
\newsavebox{\boxtwoa}
\newsavebox{\boxthreea}
\begin{figure*}
\hspace{25pt}
\begin{minipage}{.45\textwidth}
\begin{lrbox}{\boxonea}%
\begin{lstlisting}[identifierstyle=, keywordstyle=,stringstyle=]
set var_8 to "str_2";
instantiate var_2;
read var_2 from stdin;
add "str_4" to the beginning of var_8;
set var_2 to false;
set var_5 to true;
if var_2 is true, swap the values of var_2 and var_5;
if var_2 is true, set var_2 to the value of var_5 and var_5 to the value of var_2;
output var_8 to stdout;
print var_2;
output var_5 to stdout;
\end{lstlisting}
\end{lrbox}
\scalebox{0.75}{\usebox{\boxonea}}
\end{minipage}
\hspace{-48pt}
\begin{minipage}{.35\textwidth}
\hspace{8pt}
\begin{lrbox}{\boxtwoa}%
\begin{lstlisting}
int main () {
  string var_8 = "str_2";
  bool var_2;
  cin >> var_2;
  var_8 = "str_4" + var_8;
  var_2 = false;
  `var_5 = true;`
  if ( var_2 ) {
    `string temp = var_2;`
    var_2 = var_5;
    var_5 = temp; }
  if ( var_2 ) {
    bool temp = var_2;
    var_2 = var_5;
    var_5 = temp; }
  cout << var_8;
  cout << var_2;
  cout << var_5;
  return 0; }
\end{lstlisting}
\end{lrbox}
\scalebox{0.75}{\usebox{\boxtwoa}}
\end{minipage}
\hspace{-55pt}
\begin{minipage}{.34\textwidth}
\hspace{25pt}
\begin{lrbox}{\boxthreea}%
\begin{lstlisting}[]
int main () {
  string var_8 = "str_2";
  bool var_2;
  cin >> var_2;
  var_8 = "str_4" + var_8;
  var_2 = false;
  `bool var_5 = true;`
  if ( var_2 ) {
    `bool temp = var_2;`
    var_2 = var_5;
    var_5 = temp; }
  if ( var_2 ) {
    bool temp = var_2;
    var_2 = var_5;
    var_5 = temp; }
  cout << var_8;
  cout << var_2;
  cout << var_5;
  return 0; }
\end{lstlisting}
\end{lrbox}
\scalebox{0.75}{\usebox{\boxthreea}}
\end{minipage}
\caption{
    \textbf{(Left)} Example input pseudocode in \sanstype, which does not supply type information.
    \textbf{(Middle)} Output of the base predictor given the input pseudocode, which contains some errors.
    \textbf{(Right)} Output of the denoiser, which instantiates \texttt{var\_5} and corrects the type of \texttt{temp}. Enforcing type consistency requires resolving complex dependencies across the entire program.
}
\label{fig:code-example-denoise}
\end{figure*}

\section{Standard fine-tuning destroys some pre-trained output structure}
\label{sec:destroy}

In this section, we show that standard fine-tuning is suboptimal.
To do so, we consider pseudocode-to-code translation, where the input space $\sX$ is pseudocode
and the set of valid outputs $\sV$ is code that compiles and executes.

We consider \sanstype, a synthetic pseudocode-to-code dataset we generated from pseudocode and code
templates (see Figure~\ref{fig:code-example-denoise} for an example).
In \sanstype, the pseudocode does not supply types (e.g., \texttt{set x to 5}) and thus the model must look at how \texttt{x} is used in the rest of the generated program to generate correct code (e.g., \texttt{int x = 5} if \texttt{x} has not been previously initialized).
For this experiment, we pre-train a denoiser with synthetic corruptions that delete and substitute tokens from unlabeled code examples (see Appendix~\ref{app:synthetic}).

\begin{table}[t]
\centering
\scalebox{0.9}{
\begin{tabular} {l r r r r r}
\toprule
& \multirow{2}{*}{\# Layers} & \multirow{2}{2cm}{Test-time denoiser?} & \multirow{2}{*}{Correct}\\\\
\midrule
Std fine-tuning & 6 & N & 79.8\\
Std fine-tuning & 6 & Y & 80.4\\
Std fine-tuning & 12 & N & 79.8\\
Std fine-tuning & 12 & Y & 81.2\\
\bottomrule
\end{tabular}
}
\caption{
    Proportion of correct code (\%) which compiles, executes, and passes test cases in the \sanstype~task. Applying the pre-trained denoiser to the outputs of the fine-tuned model improves the performance, suggesting that standard fine-tuning does not fully leverage the pre-trained parameters.
}
\label{table:destroy}
\end{table}

Table~\ref{table:destroy} shows the results of standard fine-tuning on \sanstype~for 6 and 12 layer models.
We find that directly applying the pre-trained denoiser $\Pi$ to the
outputs of the fine-tuned model $\fstd$ at test-time (i.e., a test-time
denoiser) improves correctness of the output code by
0.5--1.5\%, and this gap increases with the number of layers.
Note that increasing the number of layers does not improve the performance of the standard fine-tuned model $\fstd$, but improves the performance of $\Pi \circ \fstd$.
Thus the denoiser $\Pi$ has improved with more layers, but standard fine-tuning fails to leverage the improved denoiser.

\section{Composed fine-tuning}
\label{sec:framework}

Based on the observations in the previous section, we introduce composed fine-tuning.
We assume the pre-trained denoiser $\Pi$ with parameters $\beta$ is given as a probabilistic
model $p_\beta$ and pre-trained on unlabeled output data.
The goal is to learn the parameters $\theta$ of a base predictor $\fthetabase$ defined by a probabilistic model $p_\theta$.
After training, we use the composed predictor $\Pi \circ \fthetabase$ to make predictions.

To learn the parameters $\theta$, we maximize two log-likelihood terms for the composed and base models respectively:
\begin{align}
    \label{eqn:composed-objective}
    \E_{x,y}[ \E_{y’ \sim p_\theta(\cdot \mid x)}[\log p_\beta (y \mid y’)]] + \lambda \E_{x,y}[ \log p_\theta( y \mid x) ].
\end{align}
For the first term, we would ideally maximize the log-likelihood of a denoised output $y$, which is generated from a Markov chain $x\rightarrow y' \rightarrow y$ conditioned on $x$. In this chain, the noisy output $y'$ is sampled from the probabilistic model for the base predictor $p_{\theta}(y' \mid x)$, and $y$ is sampled from the probabilistic model for the denoiser $p_{\beta}(y \mid y')$. Computing the log-likelihood exactly involves an intractable integral over noisy outputs $y'$, and thus the first term in Equation~\ref{eqn:composed-objective} is a lower bound instead (see Appendix~\ref{app:objective}).
For discrete outputs, the first term in the objective involves
an expectation over a discrete space, which requires REINFORCE~\cite{williams1992simple}, straight-through
estimation~\cite{bengio2013estimating}, or a Gumbel-softmax
reparameterization~\cite{jang2017categorical,maddison2016concrete} to optimize.
For intuition, we can analogize this to the RL setting by viewing the base
predictor as a policy, the output space as the action space, and the
pre-trained denoiser as the reward function.
In our implementations, we use REINFORCE.

The second term is the log-likelihood of only $p_\theta$, which is the same as in standard fine-tuning.
Since the pre-trained $\Pi$ is imperfect, the hyperparameter $\lambda$
in the objective trades off between fitting the composition $\Pi \circ
\fthetabase$ and fitting $\fthetabase$ directly to the data.

\paragraph{Reducing complexity of the base predictor.}
By relying on the denoiser to correct errors, the base predictor can learn a simpler function.
For example, the base predictor could translate the pseudocode in
\sanstype~without enforcing type consistency, leading to a simpler base predictor.
Figure~\ref{fig:code-example-denoise} (left) gives an example
pseudocode input in \sanstype. With composed fine-tuning, the base
predictor could learn to make code translations on a mostly local,
line-by-line basis (a simpler solution) while relying on the denoiser
to correct types, which requires resolving complex global dependencies.

\paragraph{Adaptability to different denoisers.}
How does the choice of denoiser affect composed fine-tuning?
Naively, denoisers that are pre-trained with corruptions that are similar to errors that the base predictor would naturally make would make them more effective denoisers.
However, we note that training the base predictor composed with the denoiser allows
the base predictor to adapt to the errors the denoiser can correct. In
Section~\ref{sec:pretraining-obj} and Appendix~\ref{app:fonts}, we show that composed fine-tuning gives gains across a variety of noising distributions in both text and image domains.
We also show in Section~\ref{sec:pretraining-obj} that a pre-trained denoiser that is worse at correcting errors in the outputs of a baseline model can actually improve fine-tuning results downstream.

\section{Denoisers can reduce model complexity}
\label{sec:theory}

In this section, we formalize the intuition that composed fine-tuning reduces
the complexity of the base predictor.
We study standard versus composed fine-tuning from an
approximation theory standpoint and use complexity measures on predictors as
surrogates for sample complexity.
We have a true target function $\fstar: \sX \rightarrow \sV$ and a hypothesis class $\sH$ that includes $\fstar$.
For composed fine-tuning, we assume access to a denoiser $\Pi:\sY \rightarrow \sV$ which is a Euclidean
projection onto the valid output space $\sV$ (breaking ties arbitrarily).
Let $\|\cdot\|$ be a norm on the hypothesis class $\sH$.
We compare the minimum norm base predictor $\fthetabase\in \argmin_{f\in\sH}\{\|f\|:\Pi \circ f(x)=\fstar(x),x\in\sX\}$ trained
via composition with a given denoiser $\Pi$, against the standard
predictor $\fstd\in \argmin_{f\in\sH}\{\|f\|:f(x)=\fstar(x),x\in\sX\}$,
a minimum norm predictor which represents $\fstar$ directly, as in
standard fine-tuning. Since we consider the minimizers, we do not
capture the effects of optimization.

In this section, we consider the simplest form of output validity structure, which is a discrete set of valid outputs $\sV$.
Intuitively, the base predictor can capture the smooth global trends while the denoiser takes care of the projection onto the discrete valid set.
This allows the base predictor in composed fine-tuning to be much simpler than the standard fine-tuned model, which tries to model the discrete outputs directly.

In Section~\ref{sec:simple}, we give a simple example for when
composing with a denoiser ($\Pi \circ \fthetabase$) can drastically
reduce the complexity of the base predictor.
Since $\fthetabase$ becomes easier to approximate, we may expect better
generalization~\cite{bartlett2017spectral,neyshabur2017generalization,wei2020improved,wei2019data}.
In Section~\ref{sec:relu}, we theoretically show for two-layer ReLU
networks that the complexity required to directly represent $\fstar$
can be arbitrarily larger than representing with a composed predictor
depending on the complexity of $\fstar$, for a family of target functions defined on a discrete valid set $\sV$.

\subsection{Motivating example}\label{sec:simple}
Figure~\ref{fig:simple} shows a staircase function $\fstar$ that can be represented by a complex standard predictor $\fstd$ but the minimum norm base
predictor $\fthetabase$ can represent $\fstar$ with low complexity when composed with a denoiser.
For $0<\delta<1$, let the input space $\sX=\uplus_{i=1}^N
[i-(1-\delta)/2, i+(1-\delta)/2]$ be a union of $N$ disjoint intervals
and the valid outputs $\sV=\mathbb{Z}$ be the integers, a subset of the
ambient output space $\sY=\R$.
The staircase function is $\fstar(x) = \lfloor x \rceil$ defined on
$\sX$, which rounds a linear function onto the integers.
Following \citet{savarese2019function}, we define the norm of a
univariate function $f:\R\rightarrow \R$ as
\begin{align}
\label{eqn:norm}
\|f\| = \frac{1}{2}\max\left(\int_{-\infty}^\infty |f''(x)|^2dx, |f'(-\infty) + f'(+\infty)|\right).
\end{align}
Intuitively, this norm measures how ``wobbly'' $f$ is by summing up the changes in its second derivative.

Consider representing $\fstar$ with linear splines, a family of
piecewise linear functions.
In linear splines, the norm in Equation~\eqref{eqn:norm} becomes
roughly the sum of absolute changes in slope between piecewise
segments.
If we represent $\fstar$ directly with a linear spline $\fstd$, the
norm of $\fstd$ has to be large due to the large number of slope
changes: $\| \fstd\| = (N-1)/\delta$ (Figure~\ref{fig:simple} left).

\paragraph{Reduced complexity.}
Suppose we have access to a denoiser $\Pi(y)=\lfloor y \rceil$, which
projects onto $\sV=\mathbb{Z}$.
Then a linear function $\fthetabase$ composed with $\Pi$ can
represent the staircase on $\sX$, reducing the norm to 1
(Figure~\ref{fig:simple} right).
By not requiring $\fthetabase$ to represent the local complexity and
discreteness in $\fstar$, the base predictor $\fthetabase$ better
captures the underlying globally linear structure of $\fstar$.

\paragraph{Improved extrapolation.}
Beyond improving generalization for inputs drawn from the same distribution as
the training data, we may also consider out-of-distribution (OOD)
generalization to inputs drawn from a different distribution, where machine
learning models have been shown to fail (e.g., simple
corruptions~\citep{hendrycks2019benchmarking} and shifts across
topics~\citep{fisch2019mrqa,yogatama2019learning,gururangan2020don}).  This
simple example suggests that composed fine-tuning can improve
out-of-distribution (OOD) performance since the simpler linear base predictor
also enables perfect extrapolation outside the training domain.

\subsection{Analysis for 2-layer ReLU networks}
\label{sec:relu}

We now extend our analysis from splines to neural networks and high dimensional (discrete) outputs.
Formally, we take the valid set $\sV = \{y^*_1,\dots,y^*_N\}$ to be a
discrete set over $N$ output vectors in $\R^k$ and $\fstar$ is a
piecewise constant function defined on $N$ disjoint intervals $\sX =
\uplus_{i=1}^N [x^l_i, x^u_i]$ (in ascending order),
where there is a $\delta>0$ gap between each interval and the next.
The target function $\fstar$ is defined such that if $x\in [x^l_i,
x^u_i]$, then $\fstar(x) = y^*_i$.
We leave analysis of discrete, high-dimensional input spaces (such as in text tasks) to future work.

We study 2-layer ReLU networks, often studied as a first step towards
understanding the expressivity of neural
networks~\cite{neyshabur2014implicit,savarese2019function,eldan2016depth}.
Following~\citet{savarese2019function}, we define the 2-layer ReLU network $\ftheta \in \sH$ with parameters $\theta$ as
\begin{align*}
    \ftheta(x) = \sum_{l=1}^h w_l^{(2)} \left[\langle w_l^{(1)}, x\rangle + b_l^{(1)} \right]_{+} + b_l^{(2)}
\end{align*}
on $x\in \R^d$, where we will take $d=1$ throughout. Here, $[x]_+ =
\max(x,0)$ is the element-wise ReLU nonlinearity.
We let $W^{(1)}\in \R^{h\times d}$ denote
the first layer weight matrix with $w_l^{(1)} \in \R^d$ as rows.
Let $b^{(1)}, b^{(2)},
w^{(2)}\in\R^h$ be bias and weight vectors with $b_l^{(1)}, b_l^{(2)}, w_l^{(2)}\in\R$
as elements respectively.
The parameters $\theta=(h, W^{(1)},b^{(1)}, b^{(2)},
w^{(2)})$ contain the hidden unit size $h\in \mathbb{N}$
and all weights and biases.
We let $\Theta=\{\theta: h\in \mathbb{N}; W^{(1)}\in \R^{h\times d}; b^{(1)}, b^{(2)},
w^{(2)}\in\R^h\}$ denote this parameter space.

\paragraph{Measure of complexity.}
Following~\citet{savarese2019function}, the complexity of a network is
associated with the squared Euclidean norm of the weights
\begin{align*}
    C(\theta) = \frac{1}{2}(\|w^{(2)}\|^2_2 + \|W^{(1)}\|^2_F).
\end{align*}
The norm of $f$ is the minimum norm required to represent it:
\begin{align}
    \label{eqn:norm-nets}
    \|f\| = \inf_{\theta \in \Theta} C(\theta) \text{~s.t.~} f_{\theta} = f.
\end{align}
\citet{savarese2019function} showed that this norm is equivalent to
Equation~\ref{eqn:norm} for univariate networks with 1-dimensional inputs and
outputs.
Since these complexity measures typically appear in generalization
bounds~\cite{bartlett2017spectral,neyshabur2017generalization}, we
expect to improve generalization error by reducing these complexity
measures.

\paragraph{Minimum complexity reduces with a denoiser.}
Given $\Pi(y) \in \argmin_{y^*\in \sV} \|y^* - y\|_2$ which is
a Euclidean projection onto $\sV$ (breaking ties arbitrarily), we want to compare
the norms of $\fstd$ that represents $\fstar$ directly and the minimum
norm base predictor $\fthetabase$ that represents $\fstar$ when composed with the denoiser $\Pi$:
\begin{align}
    \fthetabase = \argmin_{f\in\sH}\{\| f\| : \Pi \circ f (x) = \fstar(x), x\in\sX\}.
\end{align}
Note that $\|\fthetabase\|\leq \|\fstd\|$ since $\fstd$ is a
feasible solution. Thus composing cannot increase the norm.

\paragraph{Adjacent intervals measure complexity of the target function.}
Our result depends crucially on the number of \emph{non-adjacent} pairs
of intervals in $\fstar$.
Suppose the output dimension is $k=1$.
We define a pair of interval indices $(i, i+1)$ as $\emph{adjacent}$ if
there is no valid output value $y \in \sV$ in between $y^*_{i}$ and
$y^*_{i+1}$; that is, none of  $y^*_{i} < y < y^*_{i+1}$ or $y^*_{i} >
y > y^*_{i+1}$ hold.
The number of non-adjacent interval pairs characterizes the complexity
of $\fstar$.
Let $|J|$ be the number of non-adjacent pairs and $|I|$ be the number
of adjacent pairs, where $|I|+|J|=N-1$.
Our bound also depends on $L = \min_i |y^*_{i}-y^*_{i+1}|$ and $U =
\max_i |y^*_{i} - y^*_{i+1}|$, the min and max separation between valid
points.
For higher output dimensions ($k>1$), let $y^*_{i,j}$ be the $j$-th
output coordinate of the $i$-th valid point and let
$|J_j|,|I_j|,L_j,U_j$ be the analogous quantities for each output
coordinate $j\in[k]$.
\begin{theorem}
\label{thm:main}
Let the valid output space $\sV\subset \R^k$ be a set over $N$
  multivariate output vectors $\{y^*_1,\dots,y^*_N\}$ in $\sV$.
Let $\fstar:\R\rightarrow \R^k$ be a piecewise constant function
  defined on $\sX = \uplus_{i=1}^N [x^l_i, x^u_i]$ where
  $\fstar(x)=y^*_i$ if $x\in[x^l_i,x^u_i]$.
Let $\Delta_x$ be the length of the smallest interval in $\sX$.
Then
\begin{align}
    \frac{\|\fstd\|}{\|\fthetabase\|} = \Omega \left(\frac{N\max_jL_j}{\sum_{j=1}^{k}U_j\left(|J_j|+\delta\frac{|I_j|}{\Delta_x}\right)}\right)
\end{align}
\end{theorem}
See Appendix~\ref{app:analysis} for a proof.
The ratio $\frac{\|\fstd\|}{\|\fthetabase\|}$ is the complexity savings from composed fine-tuning.
In the theorem, $|J_j|$ measures the inherent complexity of the target function coordinate
$j$. If $|J_j|$ are sublinear in $N$ (i.e., the target function is simple) and valid points are evenly
spaced, then the ratio is $\Omega(1/\delta)$, which can be arbitrarily
large for a fixed output dimension as $\delta\rightarrow 0$ or $N\rightarrow \infty$.
If any $|J_j|$ is linear in $N$ (i.e., the target function is complex), then there is only a constant factor ratio in
the worst case.
Overall, if $\fstar$ is not too complex (has small $|J_j|$ values) with respect to its discrete output
space, we can learn a simpler base predictor that still represents $\fstar$ when composed with the denoiser.

\section{Experiments}

We show on two pseudocode-to-code datasets that composed fine-tuning improves generalization over standard fine-tuning (Section~\ref{sec:code}).
Composed fine-tuning also boosts generalization when combined with other fine-tuning methods such as backtranslation, which has additional access to a large unlabeled dataset during fine-tuning (Section~\ref{sec:ablations}).
We also show that composed fine-tuning can lead to better OOD performance on shifted pseudocode distributions and unseen programs in pseudocode-to-code (Section~\ref{sec:code}), as well as unseen inputs in an image generation task (Section~\ref{sec:fonts}).

\subsection{Main results}
\label{sec:code}
We evaluate composed fine-tuning on two pseudocode-to-code datasets,
\sanstype~and \spoc.

\paragraph{Prediction task and pre-trained model.}
We consider full-program pseudocode-to-code translation.
The input space $\sX$ is pseudocode, the ambient output space $\sY$ is the set of strings, and the set of valid outputs $\sV$ are strings
that compile with the \texttt{g++} compiler.
In contrast to previous works which decompose the problem into
line-by-line translation and use information from the
compiler~\cite{kulal2019spoc,yasunaga2020repair}, we translate the
entire program at once without compiler access at test time.
Following~\citet{yasunaga2020repair}, the pre-training objective for
both pseudocode-to-code datasets consists of repairing random semantic
and syntactic corruptions of unlabeled code examples (using some domain knowledge of code).
See Appendix~\ref{app:synthetic} and Appendix~\ref{app:spoc} for pre-training details for each dataset.

\paragraph{Models for \sanstype.}
We use Transformer encoders/decoders~\cite{vaswani2017attention} for the base predictor and the denoiser (see Appendix~\ref{app:synthetic}).
The baseline model, standard fine-tuned model, base predictor, and denoiser are Transformer encoder-decoders~\cite{vaswani2017attention} where the encoder and decoder each have 3 (or 6) layers, 2 attention heads, embedding dimension 256, and fully connected embedding dimension 1024.
The base predictor in composed fine-tuning is initialized from the standard fine-tuned model.
Models trained from scratch are trained for 500 epochs. For composed fine-tuning, the denoiser is pre-trained for 50 epochs, and we fine-tune for 300 epochs.
In all models, we use weight decay 0.0001, dropout probability 0.1, attention dropout probability 0.2, and ReLU dropout probability 0.2 as regularization.

\paragraph{Models for \spoc.}
We follow similar protocols for \spoc with different architectures and training details (see Appendix~\ref{app:spoc}).
For \spoc, all models are Transformer encoder-decoders with 5 layers, 8 attention heads, embedding dimension 256, and fully connected embedding dimension 1024.
We use dropout probability 0.4, attention dropout probability 0.2, ReLU dropout probability 0.2, weight decay 0.0001, following~\citet{guzman2019flores}.
For standard fine-tuning, we use a decaying label smoothing schedule with smoothing parameter starting with 0.2 for 100 epochs, then 0.1 and 0.05 for 25 epochs each.
We found that reducing the label smoothing parameter near the end of training improves generalization for all models.
For the denoiser, we also use a decaying label-smoothing schedule with smoothing parameter starting at 0.2 for 5 epochs, then decaying to 0.1 for 1 epoch, and finally 0.0 for 3 epochs.
For composed fine-tuning, we initialize the base predictor from the standard fine-tuning model and train for 20 epochs.

\paragraph{Hyperparameters.}
We use $\lambda=1$ to balance between the fitting the composed and direct objectives.
During inference, we use greedy decoding for simplicity, although
 beam search or compiler access could improve the results further.
We use the validation loss for early stopping for all models, and we chose the maximum number of epochs by inspecting training loss convergence.
 For composed fine-tuning, we specifically use the validation loss of the composition $\Pi \circ \fthetabase$ for early stopping.

\paragraph{Metrics.}
We report the fraction of examples that successfully compiled as well as the fraction of examples that are correct (accuracy).
A program is \emph{correct} when executed on a set of input test cases if, on each input test case, the program output matches the corresponding gold output.

\subsubsection{\sanstype}
\label{sec:synthetic}
We introduce \sanstype, a synthetic pseudocode-to-code dataset where the pseudocode specifies all but the declaration types, leaving global type consistency to the model (see
Figure~\ref{fig:code-example-denoise}).
Composed fine-tuning can avoid modeling everything
by learning simple a base predictor $\fthetabase$ that can do local translation
while the denoiser $\Pi$ enforces global constraints such as type
consistency.

\begin{table}
\centering
\scalebox{0.8}{
\begin{tabular} {l r r r r}
\toprule
 & \# Layers & Compiled & Accuracy & Relative\\
\midrule
\sanstype & & & & \\
~~~~Baseline & 6 & 48.6 & 36.6 & -54.1 \\
~~~~Standard fine-tuning & 6 & 86.8 & 79.8 & 0.0 \\
~~~~Composed fine-tuning & 6 & \textbf{90.4} & \textbf{84.4} & \textbf{5.8}\\
\sanstype~(OOD) & & & & \\
~~~~Baseline & 6 & 37.2 & 14.2 & -66.5 \\
~~~~Standard fine-tuning & 6 & 65.0 & 42.4 & 0.0 \\
~~~~Composed fine-tuning & 6 & \textbf{81.6} & \textbf{52.8} & \textbf{24.5}\\
\midrule
\spoc~\testw & & & & \\
~~~~Baseline & 10 & 51.3 & 34.5 & -6.8  \\
~~~~Standard fine-tuning & 10 & 52.5 & 37.0 & 0.0 \\
~~~~Composed fine-tuning & 10 & \textbf{53.9} & \textbf{38.1} & \textbf{3.0} \\
\spoc~\testp~(OOD) & & & & \\
~~~~Baseline & 10 & 24.5 & 12.2 & -17.6 \\
~~~~Standard fine-tuning & 10 & 24.7 & 14.8 & 0.0 \\
~~~~Composed fine-tuning & 10 & \textbf{25.8} & \textbf{15.4} & \textbf{4.1} \\
\bottomrule
\end{tabular}
}
\caption{
    Proportion of generated code (\%) that compiled successfully and that
    executed with a correct output (accuracy) in the \sanstype~and
    \spoc~pseudocode-to-code tasks. Last column is the relative improvement in
    accuracy over standard fine-tuning.
}
\label{table:code-main}
\end{table}

\paragraph{Dataset generation.} The synthetic programs involve 1--4
variables (bools, ints, and strings) drawn from 10 possible variable
names, which are first initialized (by reading stdin) and then
processed by up to 1--5 random operations, including 3 unary operations
per type and 2 binary operations on ints (see Appendix~\ref{app:synthetic}). There are 100 possible
integer values and 10 possible string values. Pseudocode for each line
of code is generated from a few possible templates.
We generate 1000 labeled examples and 20000 unlabeled code examples.

\paragraph{In-distribution and OOD test sets.}
The in-distribution test set has pseudocode generated from the same
templates as the training set. The OOD test set uses a new set of
pseudocode templates created by mixing and matching tokens from the
in-distribution templates.
For example, for the print statement \texttt{cout << <var>}, the
in-distribution pseudocode templates include \texttt{print <var>} and
\texttt{output <var> to stdout} while the OOD pseudocode contains the
templates \texttt{print <var> to stdout}, \texttt{output <var>}, and
\texttt{stdout <var>}.
Although perhaps a benign shift, this is enough to significantly lower the
performance of all models.

\subsubsection{\spoc} \label{sec:spoc}
We also evaluate on the challenging \spoc~pseudocode-to-code
dataset~\cite{kulal2019spoc}, which contains code scraped from
\texttt{codeforces.com} and pseudocode written by crowdsourced workers.
Since we consider the full-program translation task instead of
line-by-line as in previous
works~\cite{kulal2019spoc,yasunaga2020repair}, we filter out training
examples where the code is longer than 1000 tokens after
pre-processing, retaining over 95\% of the training examples.

\paragraph{In-distribution and OOD test sets.}
We use the two given \spoc~test sets, \testw~and \testp.
\testw~contains examples with pseudocode written by different crowdworkers on the same coding problems as the training set, while \testp~contains examples from unseen coding problems.
We consider \testp~to be the out-of-distribution
test set.
We report results on the full (unfiltered) test sets.

\begin{figure}[tbp]
  \centering
  \subfloat[Vary labeled data]{\includegraphics[width=0.36\textwidth]{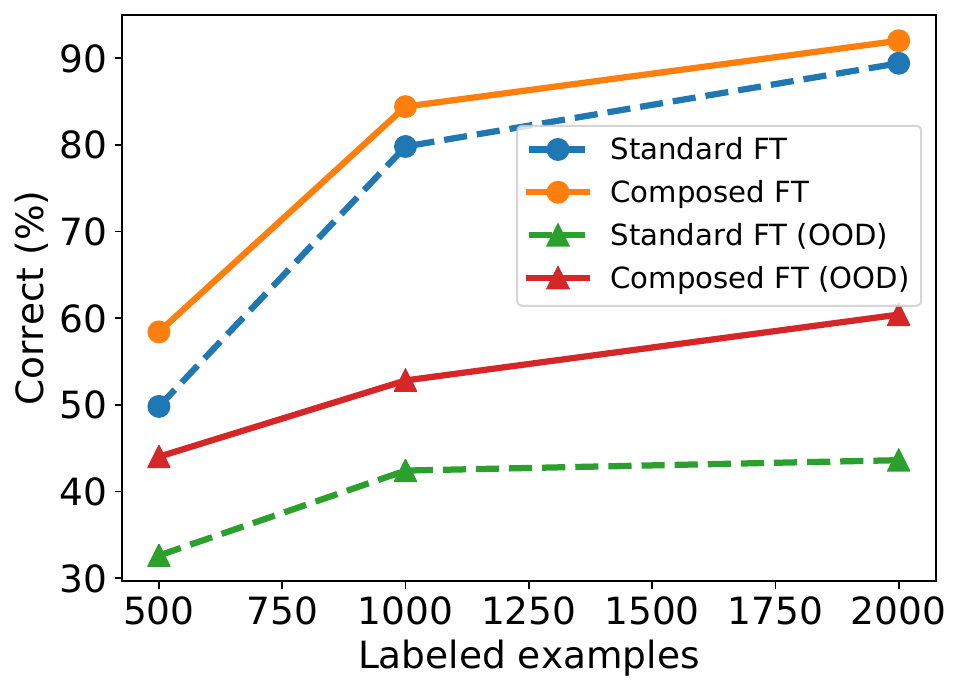}}
  \subfloat[Vary unlabeled data]{\includegraphics[width=0.36\textwidth]{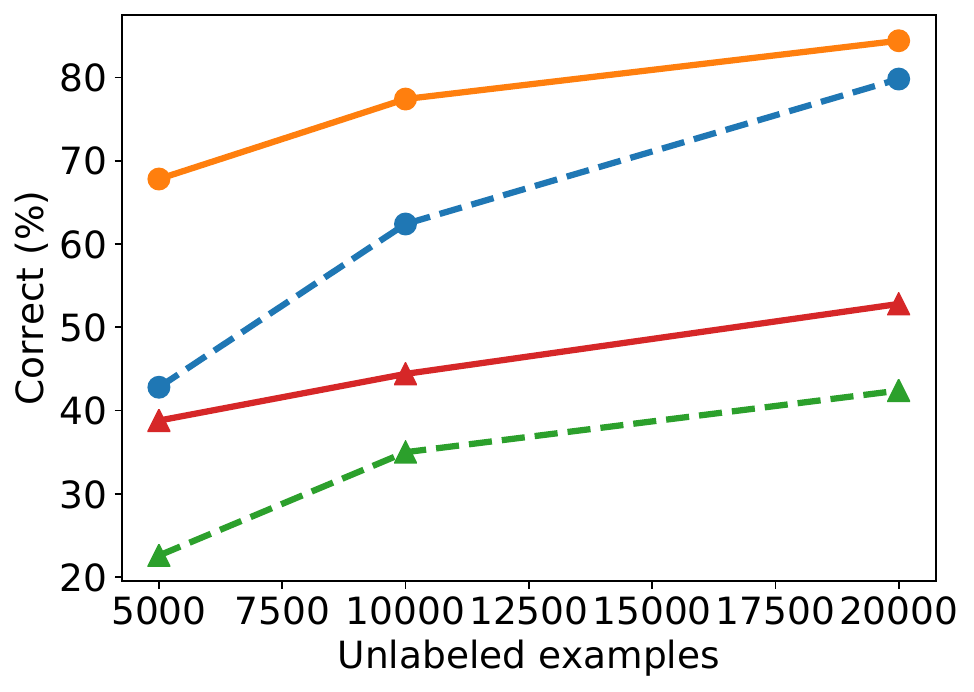}}
  \caption{Performance with varying labeled and unlabeled data in \sanstype. Composed fine-tuning improves more over standard fine-tuning as sample size decreases.  }\label{fig:synthetic-labeled-unlabeled}
\end{figure}

\subsubsection{Results}
Table~\ref{table:code-main} shows the results for a baseline
Transformer trained from scratch, standard fine-tuning, and composed
fine-tuning for \sanstype~and \spoc.
Composed fine-tuning gives about 6\% and 3\% relative improvement over
standard fine-tuning for \sanstype~and \spoc~respectively on
in-distribution test sets. On OOD test sets, the relative improvement
grows to 25\% and 4\%, respectively.
Thus, composed fine-tuning improves generalization on in-distribution
examples and particularly out-of-distribution examples.
Figure~\ref{fig:synthetic-labeled-unlabeled} shows results on varying
unlabeled data (used for pre-training the denoiser) and labeled data sizes, where composed fine-tuning improves
more over standard fine-tuning as sample size decreases, and the larger
relative improvement in the OOD split holds across all sample sizes.

Figure~\ref{fig:code-example-denoise} gives an example \sanstype~input
with the output of the base and composed predictors. With the denoiser,
the base predictor does not have to output all the correct variable
types.
Here, the denoiser correctly instantiates \texttt{var\_5} and corrects the type of \texttt{temp}.

\subsection{Ablations}
\label{sec:ablations}

In this section, we compare to standard fine-tuning from a larger pre-trained denoiser (Section~\ref{sec:larger}), ablate the composed training step (Section~\ref{sec:testtime}), combine composed fine-tuning with other fine-tuning methods (Section~\ref{sec:othermethods}), and investigate the effect of different pre-trained denoisers (Section~\ref{sec:pretraining-obj}).
We find that composed training is important for the gains from composed fine-tuning and that composed fine-tuning is complementary with other fine-tuning methods. Composed fine-tuning also works well for denoisers trained with simple synthetic corruptions that do not have any domain knowledge about code.

\subsubsection{Standard fine-tuning from a larger denoiser}
\label{sec:larger}
We consider standard fine-tuning from a 12 layer denoiser, which is double the size of the denoiser in the composed fine-tuning models. Although the total number of layers is the same as composed fine-tuning models, this is not a fair comparison since the pre-trained denoiser is much larger (benefits from the large unlabeled set more) and the number of learned parameters is double that of composed fine-tuning.

Table~\ref{table:synthetic-code-2} shows the results for larger baseline and standard fine-tuning models (with 12 layers) on \sanstype.
On the in-distribution \sanstype~split, we find that composed fine-tuning (84.4\%) still improves over standard fine-tuning (79.8\%) in this setting. However, on the OOD \sanstype~split, standard fine-tuning with the larger denoiser (58.4\%) has higher accuracy than composed fine-tuning (52.8\%). Increasing the size of the pre-trained denoiser seems to give a large boost for downstream OOD accuracy. However, with a much smaller pre-trained model, composed fine-tuning still outputs a higher proportion of successfully compiled code (81.6\%) than the larger standard fine-tuning model (74.0\%).
On \spoc, we find that using more layers (increasing from 10 to 20 layers) does not significantly change the baseline performance.
While the compilation rate increases slightly, increasing the number of layers slightly degrades the accuracy of its output programs.

\begin{table}[tbp]
\centering
\scalebox{0.8}{
\begin{tabular} {l r r r r}
\toprule
 & \multirow{2}{*}{\# Layers} & \multirow{2}{1.5cm}{Test-time denoiser?} & \multirow{2}{1.2cm}{Compiled} & \multirow{2}{*}{Accuracy}\\\\
\midrule
~~~~Composed fine-tuning & 6 & Y & \textbf{90.4}  & \textbf{84.4}\\
Larger denoiser & & & & \\
~~~~Baseline & 12 & N  & 58.6 & 41.6 \\
~~~~Standard fine-tuning & 12 & N  & 87.0 & 79.8 \\
Test-time denoiser & & & & \\
~~~~Baseline & 6 & Y & 76.4 & 58.4 \\
~~~~Standard fine-tuning & 6 & Y & 87.3 & 80.4 \\
\midrule
~~~~Composed fine-tuning (OOD) & 6 & Y & \textbf{81.6}  & 52.8\\
Larger denoiser (OOD) & & & & \\
~~~~Baseline & 12 & N  & 48.2 & 22.4 \\
~~~~Standard fine-tuning & 12 & N  & 74.0 & \textbf{58.4} \\
Test-time denoiser (OOD) & & & & \\
~~~~Baseline & 6 & Y & 64.8 & 29.0 \\
~~~~Standard fine-tuning & 6 & Y & 76.2 & 49.8 \\
\bottomrule
\end{tabular}
}
\caption{
    Results of standard fine-tuning with access to a larger denoiser and baselines with a test-time denoiser on \sanstype. On \spoc, we find that does not significantly change the baseline performance and that using a test-time denoiser does not improve accuracy of \spoc~baselines.
}
\label{table:synthetic-code-2}
\end{table}

\subsubsection{Comparisons to baselines with a test-time denoiser}
\label{sec:testtime}
We compare to baselines where we apply the pre-trained denoiser at test time (a test-time denoiser). With the test-time denoiser, the baselines have the same architecture and number of trainable parameters as the composed predictor. The only difference is that composed fine-tuning composes the base predictor with the denoiser during training.
Composed fine-tuning improves over models with a test-time denoiser, highlighting the importance of training with the fixed denoiser.

Table~\ref{table:synthetic-code-2} shows the results on \sanstype.
While standard fine-tuning improves with a test-time denoiser, composed fine-tuning still improves over standard fine-tuning by 4\%.
We find that using a test-time denoiser does not improve accuracy of \spoc~baselines.

\begin{table}
\centering
\scalebox{0.8}{
\begin{tabular} {l r r r r r}
\toprule
 & \multirow{2}{*}{\# Layers} & \multirow{2}{1.5cm}{Test-time denoiser?} & \multirow{2}{1.8cm}{Unlabeled + Labeled?} & \multirow{2}{*}{Compiled} & \multirow{2}{*}{Accuracy}\\\\
\midrule
BART-style fine-tuning & 6 & N & N & 89.4 & 83.2 \\
BART-style fine-tuning & 6 & Y & N & 90.0 & 83.6 \\
BART-style fine-tuning & 12 & N & N & 87.8 & 81.2 \\
Composed BART-style fine-tuning & 6 & Y & N & \textbf{91.4} & \textbf{85.4}\\
\midrule
Backtranslation (BT) & 6 & N & Y & 96.4 & 86.4 \\
Backtranslation (BT) & 6 & Y & Y & 96.6 & 86.0 \\
Backtranslation (BT) & 12 & N & Y & \textbf{98.8} & 86.2 \\
Composed fine-tuning + BT & 6 & Y & Y & 96.8 & \textbf{87.0}\\
\bottomrule
\end{tabular}
}
\caption{
    Results of BART-style fine-tuning and backtranslation on \sanstype. Proportion of generated code (\%) that compiles and that results in a correct output (accuracy).
    Backtranslation uses both unlabeled and labeled data together when learning the downstream task for targeted usage of unlabeled data.
}
\label{table:other-comps}
\end{table}

\subsubsection{Composed fine-tuning improves other fine-tuning methods}
\label{sec:othermethods}
Composed fine-tuning is a general strategy. While we applied it to standard fine-tuning in the previous sections, we can also consider combining composed training with other fine-tuning methods. We show complementary gains when combining composed fine-tuning with BART-style fine-tuning~\cite{lewis2020bart} and backtranslation~\cite{sennrich2016monolingual}.

\paragraph{BART-style fine-tuning.}
BART~\cite{lewis2020bart} proposes a modified two-stage fine-tuning strategy which freezes the later layers of the encoder-decoder in the first stage and then fine-tunes the entire network in the second stage (see Appendix~\ref{app:synthetic} for details).
We apply composed fine-tuning to this by initializing the base predictor from the result of BART-style fine-tuning and doing composed training with the fixed pre-trained denoiser. Table~\ref{table:other-comps} shows that this improves further over BART-style fine-tuning (by 2--4\%), showing that the benefits of composed fine-tuning and BART-style fine-tuning are complementary.

\paragraph{Backtranslation.}
A property of pre-training is that we do not need access to unlabeled data when fine-tuning on the downstream task.
In contrast, backtranslation~\cite{sennrich2016monolingual} is a strong method that uses the large unlabeled dataset together with the labeled data in downstream learning, which is expensive but allows the method to target the usage of unlabeled data to the downstream task.
Backtranslation uses an output-to-input model (learned on the labeled data) applied on unlabeled outputs to generate additional synthetic inputs.
We employ composed fine-tuning on top by initializing the base predictor from the model trained by backtranslation, showing complementary benefits to backtranslation (by 0.6\% to 1\%) in Table~\ref{table:other-comps}.

\subsubsection{Effect of different denoisers}
\label{sec:pretraining-obj}
Naively, better pre-trained denoisers should improve the gains with composed fine-tuning.
For example, the \spoc~denoiser is not as effective as the \sanstype~denoiser, which correlates with less gains with composed fine-tuning.
Interestingly, the relationship between pre-trained denoiser quality and downstream performance is not always straightforward.
Table~\ref{table:masked-lm} shows results when we train a \sanstype~denoiser using random word deletions (no perturbations with any domain knowledge of code). The baseline with test-time denoiser worsens from 58.4\% to 46.2\% in accuracy, suggesting that the denoiser is worse when trained with the random deletion objective. However, the downstream standard and composed fine-tuning accuracies improve substantially for both in- and out-of-distribution examples.
Thus the new denoiser, which uses more generic perturbations, is worse at correcting baseline outputs but improves fine-tuning results.
The effect of pre-training on the downstream performance after adaptation merits further investigation.

\begin{table}
\centering
\scalebox{0.8}{
\begin{tabular} {l r r r r}
\toprule
 & \multirow{2}{*}{Denoiser pre-training corruptions} & \multirow{2}{*}{\# Layers} & \multirow{2}{*}{Compiled} & \multirow{2}{*}{Accuracy}\\\\
\midrule
Test-time denoiser & & & & \\
~~~~Baseline & Code-based & 6 & \textbf{76.4} & \textbf{58.4} \\
~~~~Baseline & Random deletions & 6 & 64.4 & 46.2 \\
\midrule
\sanstype & & & & \\
~~~~Standard fine-tuning & Code-based & 6 & 86.8 & 79.8\\
~~~~Standard fine-tuning & Random deletions & 6 & \textbf{92.0} & \textbf{81.2} \\
\sanstype~(OOD) & & & & \\
~~~~Standard fine-tuning & Code-based & 6 & 65.0 & 42.4 \\
~~~~Standard fine-tuning & Random deletions & 6 & \textbf{69.8} & \textbf{54.0} \\
\midrule
\sanstype & & & & \\
~~~~Composed fine-tuning & Code-based & 6 & 90.4 & 84.4\\
~~~~Composed fine-tuning & Random deletions & 6 & \textbf{94.8} & \textbf{86.2} \\
\sanstype~(OOD) & & & & \\
~~~~Composed fine-tuning & Code-based & 6 & \textbf{81.6} & 52.8 \\
~~~~Composed fine-tuning & Random deletions & 6 & 75.6 & \textbf{60.2} \\
\bottomrule
\end{tabular}
}
\caption{
    \sanstype~results on standard and composed fine-tuning when changing the pre-training denoising autoencoder objective from correcting code-based perturbations to correcting random, unstructured word deletions. The denoising quality of the random deletion denoiser is worse when applied to a baseline model, but it leads to better downstream performance both in-distribution and OOD.
}
\label{table:masked-lm}
\end{table}

\subsection{OOD generalization in image generation}
\label{sec:fonts}
\begin{figure}[tbp]
  \centering
  \subfloat[Baseline]{\includegraphics[width=0.3\textwidth]{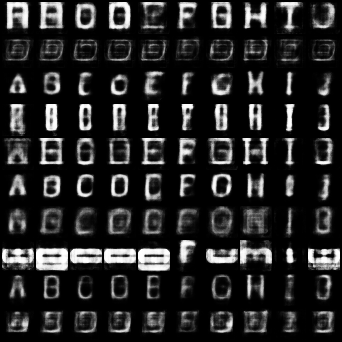}}
  \hfill
  \subfloat[Composed]{\includegraphics[width=0.3\textwidth]{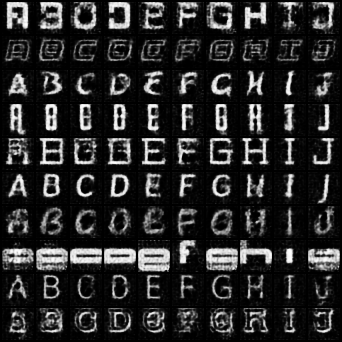}}
  \hfill
  \subfloat[Base]{\includegraphics[width=0.3\textwidth]{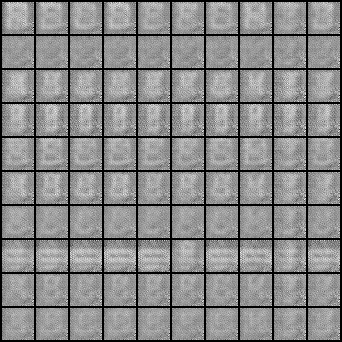}}
\caption{Generated letters A-J for 10 randomly selected fonts.
    \textbf{(a)} The baseline predictor makes blurry outputs with many artifacts.
    \textbf{(b)} The composed predictor (base + denoiser) makes clearer outputs with more distinct font patterns.
    \textbf{(c)} The base predictor produces blurrier outputs corresponding to a lower norm model.}\label{fig:fonts}
\end{figure}

We evaluate the general framework of composed training with a denoiser on font image generation from unseen attribute combinations, a natural OOD generalization task.
Generating missing characters is useful when prototyping a new font~\cite{miyazaki2020automatic,tenenbaum2000separating}. From a few user-supplied characters, the model fills in the rest.
This task closely mirrors the theoretical setup, where the input is low-dimensional (index of the font and character type) and the output is high-dimensional (image).
Qualitatively, image samples from our composed predictor are clearer and has less artifacts.

\paragraph{Prediction task.}
We map two one-hot vectors, corresponding to the character identity (out of 62 possible) and the font (out of 100 fonts), to generate $32\times 32$ grayscale font images.
Here, valid font images have clean lines and adhere to the font styling.
We test image sample quality directly by computing the pixel-wise squared error with respect to ground truth test images.

\paragraph{Data.}
We use a dataset of 56k fonts originally scraped from the Internet~\cite{bernhardsson2016fonts}.
We randomly split the 6200 labeled examples (62 characters $\times$ 100 fonts) into 2500 training examples, 100 validation examples, and 3600 test examples.
The training examples contain a random subset of the characters for each font. The models must generate the unseen characters of each font with the correct font styling at test-time.
The denoiser uses additional unlabeled images for $\sim$50k other fonts.

\paragraph{Model architectures.}
The baseline and base predictors are 7-layer fully-connected networks that map attribute encodings to images (see Appendix~\ref{app:fonts}) .
The denoiser $\Pi$ is a 3-layer U-Net~\cite{ronneberger2015unet} image-to-image model.
The denoiser is pre-trained to sharpen unlabeled font images distorted by a Gaussian blur filter with randomly sampled radii in $[0, 2]$.
We train all predictors using the pixel-wise squared error loss and tune L2 regularization strength on a validation set.
In the composed model, we set $\lambda=0$ in \eqref{eqn:composed-objective}, using only the composed loss.
We also report gains with other noising functions (embossing, contrast perturbations) in Appendix~\ref{app:fonts}.
For this task, we train all the predictors from scratch. We do not consider
fine-tuning initialized from the denoiser in this task since the predictor is
not an image-to-image model like the denoiser.

\paragraph{Results.}
The composed predictor achieves an 11\% reduction in test MSE (0.193 to 0.171) compared to the baseline predictor.
We also find that adding dropout improves the MSE further to 0.165, suggesting that dropout helps in finding a lower complexity base predictor.
We visualize the predicted images for some randomly-selected fonts for comparison (Figure~\ref{fig:fonts}).
The base predictor outputs noisy gray images, suggesting that it has learned a lower complexity model.
In contrast, directly outputting clearly defined lines and transitions between black and white pixels requires a relatively high complexity model.
Additional results on varying labeled and unlabeled data size are in Appendix~\ref{app:fonts}, where the performance of the composed model improves upon the baseline on all instances.

\section{Related work and discussion}
\label{sec:discussion}

\paragraph{Sequence-to-sequence pre-training.}
Denoising autoencoders (DAE) are classical building blocks for
unsupervised deep representation learning~\cite{vincent2008denoise,
vincent2010stacked}.
Recently, DAEs have been used for large scale sequence-to-sequence
pre-training as in BART~\cite{lewis2020bart} or
T5~\cite{raffel2019exploring}.
Machine translation methods also use monolingual data in both the
source and target languages to improve their
models~\cite{sennrich2016monolingual,cheng2016sslnlp}.
Pretraining methods use language modeling on monolingual data to
initialize the encoder and
decoder~\cite{ramachandran2018pretraining,skorokhodov2018ssl,devlin2019bert}.
Our work is complementary to these pre-training methods in that we aim
to improve the fine-tuning step to fully leverage the pre-trained
model.

\paragraph{Backtranslation.}
Back-translation methods generate additional synthetic parallel
examples by training on the backwards (target to source)
problem~\cite{sennrich2016monolingual}. While back-translation is a
strong baseline method, it requires using the large unlabeled dataset
for every downstream task. We showed that applying composed fine-tuning
to backtranslation gives complementary gains.

\paragraph{Semantic parsing and structured prediction.}
Some recent semantic parsing works have explicitly provided output constraints
using abstract syntax trees (AST) and enforcing type
constraints~\cite{yin2017syntactic,krishnamurthy2017neural,xiao2016sequence,dong2016logical}.
\citet{krishnamurthy2017neural} note that enforcing type constraints during
training not only prevents invalid outputs but also improves generalization,
supporting our results.
Structured prediction spans applications including
speech~\cite{zhang2013denoising}, vision~\cite{mueller2013semantic},
and medical diagnosis~\cite{jagannatha2016structured}. Many approaches
use graphical models (on top of neural models) for enforcing validity,
e.g. HMMs and CRFs in OCR and sequence
tagging~\cite{kassel1995comparison,huang2015bidirectional}.

These approaches typically require careful engineering to integrate the
graphical model with a neural component and do not consider the simplicity
benefits of composition.
\citet{mostajabi2018regularizing} have a similar goal of squeezing more
information out of structured output spaces. However, we also use
unlabeled output data and leverage the frozen denoiser to improve
OOD generalization. Since we perform denoising
autoencoding, we use the entire denoiser (instead of half) to correct outputs.
Work on approximate inference for structured
prediction~\cite{tu2019benchmarking} model the energy scoring and inference
processes jointly, which is possibly complex to learn. We learn the denoiser
(which like an inference algorithm in this comparison) separately on unlabeled
outputs. Composing \emph{simplifies} the base predictor for improved generalization.

\paragraph{Semi-supervised learning.}
Like semi-supervised learning, composed fine-tuning leverages large amounts of
unlabeled data through the pre-trained denoiser.  However, semi-supervised
learning works typically use unlabeled \emph{input}
data~\cite{tarvainen2017mean,miyato2018virtual,shu2018dirtt,berthelot2019mixmatch},
whereas we have ``unlabeled'' outputs.
In classification, unlabeled outputs can help with handling label
shift~\cite{lipton2018labelshift,azizzadenesheli2019reglabel}, but
otherwise there is no output structure.
If both unlabeled inputs and outputs are available, our method is
complementary with these semi-supervised methods.

\paragraph{Catastrophic forgetting.}
Catastrophic forgetting is the phenomenon that continually-trained models lose
the ability to do earlier tasks. This is related to our observation that
fine-tuning ``destroys'' some of the pre-trained output structure.
Methods for mitigating catastrophic
forgetting~\cite{french1999catastrophic,mccloskey1989catastrphic,ratcliff1990connectionist}
often aim to learn a model that does well on all tasks it has been trained on.
This motivates multi-task
learning/rehearsal~\cite{rusu2015policy,parisotto2015actor} (which uses data
from all tasks at once, training jointly) and elastic weight
consolidation~\cite{kirkpatrick2017overcoming} (which tries to partition the
network into separate experts for each task). Our goal is different: we do not
require the fine-tuned model to do the pre-training task well, and we do not
access the unlabeled data during downstream learning, making multi-task methods
inapplicable.

\paragraph{Lightweight fine-tuning.}
While we consider using the pre-trained denoiser as initialization for
the base predictor in this paper, this requires a model double the size
of the denoiser.
In generality, since the base predictor can be simpler than the standard fine-tuning predictor,
we may be able to learn a small base predictor that drastically reduces the
total number of parameters that need to be learned and stored for each
downstream task.
For example, lightweight fine-tuning methods such as adapter-tuning,
prefix-tuning, and prompt
tuning~\cite{houlsby2019parameter,li2021prefix,lester2021power} require adding
only 0.1\% more parameters to achieve comparable performance to standard
fine-tuning.
Like our work, lightweight fine-tuning works also find that freezing pre-trained parameters
tends to improve OOD results.
Our theoretical analysis perhaps provides a starting point for understanding
why OOD performance can improve with frozen parameters.

\section{Conclusion}
Many tasks in machine learning require generating outputs with rich structure (images,
text, music, proteins, etc.), for which unlabeled outputs are ubiquitous.
While pre-training on these unlabeled outputs is an effective way to learn the output structure,
we showed that standard fine-tuning fails to fully leverage the pre-trained model.
Composed fine-tuning freezes the pre-trained parameters and learns a base predictor composed with the pre-trained model,
leading to a simpler base predictor that generalizes better both in- and out-of-distribution.
Our work prompts further study into fine-tuning methods that leverge frozen pre-trained components to simplify the learning problem,
reduce the number of learned parameters, and achieve better generalization for prediction problems with rich output structure.

\section{Acknowledgements}
We thank Michi Yasunaga, Robin Jia, Albert Gu, Karan Goel, Rohan Taori, and reviewers for helpful discussions and comments. SMX is supported by an NDSEG Graduate Fellowship. The work is partially supported by a PECASE award, SDSI, and SAIL at Stanford University.

\section{Reproducibility}
All data and code for reproducing the experiments are on our \href{https://worksheets.codalab.org/worksheets/0x9187b7eee9d2453389bb63dfb45c4a89}{CodaLab worksheet (\texttt{https://worksheets.codalab.org/worksheets/0x9187b7eee9d2453389bb63dfb45c4a89})} and \href{https://github.com/p-lambda/composed_finetuning}{GitHub repository (\texttt{https://github.com/p-lambda/composed\_finetuning})}.

\newpage
\bibliography{main}

\newpage
\appendix
\onecolumn
\section{Proof of Theorem~\ref{thm:main}}
\label{app:analysis}
We recall the setting in Section~\ref{sec:theory}, which compares the norm of 2-layer ReLU networks used to represent a family of piecewise constant functions directly versus in a composed manner.
The input space $\sX \subseteq \R$ is one-dimensional and the valid output space $\sV \subset \sR^k$ is a set of discrete points in a $k$-dimensional space. We first show a result for $k=1$, then extend to higher $k$.

Suppose that the input and ambient output space are 1-dimensional ($\sX\subset \R,\sY=\R$) and we use a model $\fthetabase:\sX\rightarrow \sY$ from the family of bounded-norm two-layer ReLU neural networks $\sH$.
The valid output space $\sV$ is a discrete set of points in $\R$, and the denoiser $\Pi: \R \rightarrow \sV$ maps from real values to the nearest point in the valid set (breaking ties arbitrarily).
We show that under certain conditions on the target function $\fstar:\sX \rightarrow \sV$, we can use a small-norm base predictor composed with a denoiser ($\Pi \circ \fthetabase$) to represent $\fstar$, while directly representing $\fstar$ (without $\Pi$, as in standard fine-tuning) requires a large norm.

\paragraph{Target function family.}
The target function $\fstar$ is defined on a union of disjoint intervals $\sX = \uplus_{i=0}^{N-1} [x^l_i, x^u_i]$ where the subscript index orders the intervals such that $x^l_{i+1} - x^u_i = \delta > 0$. Thus, there is a gap $\delta$ between any two intervals. We assume that all interval lengths $x^u_i - x^l_i$ are at least $\Delta_x$.
Since $\sV$ is a discrete set of points, $\fstar$ is a piecewise constant function that takes a value $y^*_0,\dots, y^*_{N-1} \in \sV$ in each of the $N$ intervals. Each interval has a distinct value, such that $y^*_i \neq y^*_{i+1}$ for any $0\leq i\leq N-2$.
We will slightly abuse notation throughout by referring to the index $i$ of an interval as ``interval $i$''.

\paragraph{Hypothesis class and norm.}
Following~\citet{savarese2019function}, we define the hypothesis class $\sH$ such that $\ftheta \in \sH$ are
\begin{align}
    \ftheta(x) = \sum_{l=1}^h w_l^{(2)} \left[\langle w_l^{(1)}, x\rangle + b_l^{(1)} \right]_{+} + b_l^{(2)}
\end{align}
over $x\in \R^d$, where we will take $d=1$ throughout. Here, $w_l^{(1)} \in \R^d$ are rows of $W^{(1)} \in \R^{h\times d}$ and $b_l^{(1)}, b_l^{(2)}, w_l^{(2)}\in\R$ are elements of $b^{(1)}, b^{(2)}, w^{(2)}\in \R^h$ respectively. The parameter space for this hypothesis class is
\begin{align}
    \theta\in \Theta = \{(h, W^{(1)}, b^{(1)}, w^{(2)}, b^{(2)}) : h\in \mathbb{N}, W^{(1)}\in\R^{h\times d}, w^{(2)},b^{(1)},b^{(2)} \in \R^h\},
\end{align}
where the number of hidden units $h$ can be unbounded. Note that since our function family is a piecewise function with a finite number of segments, a 2-layer ReLU network with a finite number of hidden units can exactly implement the function.
Each network is associated with the squared Euclidean norm of the weights
\begin{align*}
    C(\theta) = \frac{1}{2}(\|w^{(2)}\|^2_2 + \|W^{(1)}\|^2_F).
\end{align*}
The norm associated with $f\in\sH$ is the minimum norm required to implement a given $f$:
\begin{align}
    \|f\| = \inf_{\theta \in \Theta} C(\theta) \text{~s.t.~} f_{\theta} = f.
\end{align}
Our one-dimensional result relies on the result of~\citet{savarese2019function} (Theorem 3.1) that for 2-layer ReLU networks, the norm can be rewritten as
\begin{align}
\|f\| = \frac{1}{2}\max\left(\int_{-\infty}^\infty |f''(x)|^2dx, |f'(-\infty) + f'(+\infty)|\right).
\end{align}
As a corollary, for one-dimensional functions, the minimum norm interpolant in $\sH$ has equivalent norm to the norm of a linear spline interpolation of the points~\cite{savarese2019function}.

\subsection{Lower bound on $\|\fstd\|$}
We use the norm equivalence between 2-layer ReLU networks and linear splines in one dimension to compare the norm of functions from $\sH$ that represent $\fstar$ with and without $\Pi$.
\begin{lemma}
    \label{thm:normlowerbound}
    For piecewise constant target function $\fstar$ defined by endpoints at the points $(x^l_0, y^*_0), (x^u_0, y^*_0), \dots, (x^l_{N-1}, y^*_{N-1}), (x^u_{N-1}, y^*_{N-1})$, any $\fstd \in \sH$ has norm $\|\fstd\| \geq \sum_{i=0}^{N-2} \frac{|y^*_{i+1} - y^*_{i}|}{\delta}$.
\end{lemma}
\begin{proof}
    By Theorem 3.3 in~\citet{savarese2019function}, the norm of a linear spline interpolation lower bounds $\|\fstd\|$ for any $\fstd \in \sH$. We define a linear spline interpolation $f$ of the $2N$ points as
    \begin{align}
        f(x) = \begin{cases}
            y^*_i + \gamma_{2i} (x - x^l_i) & x^l_i \leq x \leq x^u_i,~0\leq i \leq N-1\\
            y^*_i + \gamma_{2i+1} (x - x^u_i) & x^u_i \leq x \leq x^l_{i+1},~0\leq i \leq N-2\\
        \end{cases}
    \end{align}
    where $\gamma_{2i+1} = \frac{y^*_{i+1} - y^*_i}{\delta}$ for $0\leq i \leq N-2$ and $\gamma_{2i}=0$ for $0\leq i \leq N-1$.
    From~\citet{savarese2019function}, we have that $\|f\| = \max(\sum_{j=0}^{2N-3}|\gamma_{j+1} - \gamma_j|, |\gamma_0 + \gamma_{N-2}|)$, which we lower bound with the first term.
    There are $N-1$ nonzero slopes $\gamma_j$ and each one contributes to the norm twice.
    Thus $\|\fstd\| \geq \|f\| = \frac{1}{2} (2\sum_{i=0}^{N-2} \frac{|y^*_{i+1} - y^*_{i}|}{\delta})$.
\end{proof}

\subsection{Upper bound on $\|\fthetabasemin\|$}

We compute an upper bound on the norm of a min-norm base predictor $\fthetabasemin$ by construction.
Consider learning a function $\fthetabase \in \sH$ where we make predictions as $\Pi$ composed with $\fthetabase$.
For every value $y^*_i\in \sV$, let $(y^l_i, y^u_i]$ be the interval of values in $\R$ closest to $y^*_i$. Thus, if $y\in (y^l_i, y^u_i]$ then $\Pi$ maps $y$ to $y^*_i$.
Without loss of generality, we have assumed that $\Pi$ breaks ties such that $y^l_i$ does not map to $y^*_i$. 

\paragraph{Adjacent intervals.} Define interval $i$ to be \emph{adjacent} to $i+1$ if it satisfies either $y^u_i = y^l_{i+1}$ or $y^l_i = y^u{i+1}$, or equivalently, there is no target value $y^*_j \in \sV$ in the interval $(y^*_i, y^*_{i+1})$.
Considering the $i$-th pair of intervals to be $(i, i+1)$, let $I$ be the index set of adjacent pairs of intervals in $\fstar$ and $J$ the index set of non-adjacent pairs, where $|I|+|J|=N-1$.
Assume $\min_i|y^*_{i+1}-y^*_i| \geq L$, $\max_i|y^*_{i+1}-y^*_i|\leq U$ are the min and max separations between valid points.
\begin{lemma}
    \label{thm:upper-bound}
The norm of the minimum-norm base predictor $\|\fthetabasemin\|$ is upper bounded as
\begin{align}
    \|\fthetabasemin\| &\leq \max\left(\frac{|J|U}{\delta} + \frac{|I|U}{\Delta_x}, \frac{U}{\Delta_x}\right).
\end{align}
\end{lemma}
\begin{proof}
We give an explicit construction $\fhat$ in the univariate setting where the norm of $\fhat$ upper bounds $\|\fthetabasemin\|$.
We define the construction $\fhat$ via a set of points $(x_0, y_0), \dots, (x_{2N-1}, y_{2N-1})$ to linearly interpolate. For interval $1\leq i \leq N-2$, we have two different cases describing the interval's relation with its previous and next intervals:
\begin{enumerate}
    \item Same direction: if $y^*_{i-1} < y^*_i < y^*_{i+1}$, set $(x_{2i}, y_{2i})=(x^l_{2i}, y^l_{2i})$ and $(x_{2i+1}, y_{2i+1})=(x^u_i, y^u_i)$. If $y^*_{i-1} > y^*_i > y^*_{i+1}$, set $(x_{2i}, y_{2i})=(x^l_{2i}, y^u_{2i})$ and $(x_{2i+1}, y_{2i+1})=(x^u_i, y^l_i)$.
    \item Change direction: if $y^*_{i-1} < y^*_i > y^*_{i+1}$, set $(x_{2i}, y_{2i})=(x^l_{2i}, y^l_{2i}+\epsilon)$ and $(x_{2i+1}, y_{2i+1})=(x^u_i, y^l_i+\epsilon)$. If $y^*_{i-1} > y^*_i < y^*_{i+1}$, set $(x_{2i}, y_{2i})=(x^l_{2i}, y^u_{2i})$ and $(x_{2i+1}, y_{2i+1})=(x^u_i, y^u_i)$.
\end{enumerate}
We will choose $\epsilon>0$ to be a small, arbitrary value.
For the beginning and end intervals $i\in \{0, N-1\}$, we choose $(x_0, y_0),(x_{2N-2}, y_{2N-2})$ to minimize the norm of the linear spline interpolation given the other points.

We change the construction for adjacent intervals as follows:
\begin{enumerate}
    \item Adjacent to previous interval ($i > 0$): If interval $i-1$ is adjacent to $i$, we change the construction such that $x_{2i} = x^l_i - \delta/2$.
    \item Adjacent to next interval ($i < N-1$): If interval $i$ is adjacent to $i+1$, then $x_{2i+1} = x^u_i + \delta/2$ (unless case 3 occurs). If $0 < i < N-1$ and $y^*_{i-1} < y^*_i > y^*_{i+1}$, then we also set $y_{2i+1} = y^l_i$ (instead of $y^l_i + \epsilon$).
    \item Adjacent to both previous and next intervals ($0 < i < N-1$): If $(i-1, i),(i,i+1)$ are adjacent and $y^*_{i-1} < y^*_i > y^*_{i+1}$, set $x_{2i+1} = (x^u_i - x^l_i) / 2$ and $y_{2i+1} = y^l_i + \epsilon$.
\end{enumerate}
The number of non-adjacent pairs of intervals in $\fstar$ determines the complexity gap between $\|\fthetabasemin\|$ and $\|\fstd\|$.

Let $\fhat$ be the linear spline interpolation of the points $(x_0, y_0), \dots, (x_{2N-1}, y_{2N-1})$ as above, where $\Pi \circ \fhat(x) = \fstar(x)$ for $x\in \sX$ by construction. As a feasible solution, $\|\fhat\| \geq \|\fthetabasemin\|$.
We distinguish between interval segments with endpoints $(x_{2i}, y_{2i}), (x_{2i+1}, y_{2i+1})$ and interconnecting segments with endpoints $(x_{2i+1}, y_{2i+1}), (x_{2(i+1)}, y_{2(i+1)})$, for $0 \leq i \leq N-2$.
For any $i$, let $\hatalpha_{2i}$ be the slope of the interval segment and $\hatalpha_{2i+1}$ be the slope of the interconnecting segment.
For some interconnecting segments in the construction, the segment is of length zero. For these interconnecting segments, we define $\hatalpha_{2i+1} = \hatalpha_{2i}$ which does not affect the norm calculation.
The norm of the construction is
\begin{align*}
    \|\fhat\| = \frac{1}{2}\max(\sum_{i=0}^{N-2} |\hatalpha_{2i+1} - \hatalpha_{2i}| + |\hatalpha_{2(i+1)} - \hatalpha_{2i+1}|, |\hatalpha_0 + \hatalpha_{N-2}|).
\end{align*}
Notice that both differences in the first term involve an interconnecting segment.

We first bound the first term in the norm.
Suppose $(i, i+1) \in J$ is an non-adjacent pair. The contribution to the norm is
\begin{align*}
    |\hatalpha_{2i+1} - \hatalpha_{2i}| + |\hatalpha_{2(i+1)} - \hatalpha_{2i+1}| & \leq 2|\hatalpha_{2i+1}|\\
                                                                                  &\leq 2\frac{|y^*_{i+1} - y^*_i|}{\delta} \leq \frac{2U}{\delta}
\end{align*}
where in the first inequality, we note that the worst-case difference in slopes in our construction is when $\hatalpha_{2i}=0$ and $\hatalpha_{2(i+1)}=0$. The second step follows from $|\hatalpha_{2i+1}| \leq \frac{\min(|y^u_i - y^l_{i+1}|, |y^l_i - y^u_{i+1}| ) + \epsilon}{\delta}$ which is upper bounded by the second inequality for small enough $\epsilon$.
For purposes of the bound, we let $y^u_j = y^l_j + U$ for $j=\argmax_i y^*_i$ and $y^l_k = y^u_k - U$ for $k=\argmin_i y^*_i$. We can do this since the construction always `'changes direction'' with slope going to 0 as $\epsilon \rightarrow 0$ for the extremal-valued intervals.

Suppose $(i, i+1) \in I$ is an adjacent pair. Let $A$ be the event that $0 < i < N-1$ and $y^*_{i-1} < y^*_i > y^*_{i+1}$. If not $A$, the contribution to the norm is $|\hatalpha_{2(i+1)} - \hatalpha_{2i}|$ since the interconnecting segment has length zero and $\hatalpha_{2i+1} = \hatalpha_{2i}$. In this case, the contribution to the norm is
\begin{align*}
    |\hatalpha_{2(i+1)} - \hatalpha_{2i}| &\leq |\hatalpha_{2(i+1)}| + |\hatalpha_{2i}|\\
                                          &\leq \frac{|y^u_{i+1} - y^l_{i+1}|}{|x^u-x^l|} + \frac{|y^u_{i} - y^l_{i}|}{|x^u-x^l|} \leq 2\frac{U}{\Delta_x}.
\end{align*}

If $A$, we have $|\hatalpha_{2i+1} - \hatalpha_{2i}| \leq 2\epsilon / (\Delta_x/2)$ from the special case. Thus the contribution to the norm is
\begin{align*}
    |\hatalpha_{2i+1} - \hatalpha_{2i}| + |\hatalpha_{2(i+1)} - \hatalpha_{2i+1}| &\leq \frac{4\epsilon}{\Delta_x} + |\hatalpha_{2(i+1)}| + |\hatalpha_{2i+1}|\\
                                                                                  &\leq \frac{4\epsilon}{\Delta_x} + \frac{|y^u_{i+1} - y^l_{i+1}|}{|x^u-x^l|} + \frac{\epsilon}{|x^u-x^l|/2}\\
                                                                                  &\leq \frac{U + 6\epsilon}{\Delta_x} \leq 2\frac{U}{\Delta_x}
\end{align*}
for small enough $\epsilon$.

For the second term in the norm, we bound
\begin{align}
    |\hatalpha_0 + \hatalpha_{N-2}| \leq |\hatalpha_0| + |\hatalpha_{N-2}| \leq \frac{2U}{\Delta_x}
\end{align}
for small enough $\epsilon$.
Putting cases together and using $\|\fthetabasemin\| \leq \|\fhat\|$ gives the result.
\end{proof}

\subsection{Univariate result}

\begin{lemma}
    \label{thm:univariate-app}
    Let $\fstar$ be a piecewise constant function defined on $\sX = \uplus_{i=1}^n [x^l_i, x^u_i]$ with values in a discrete set $\sV \subset \R$, and let $\delta>0$ be the gap between any two intervals in $\sX$.
    Let $\fstd$ be a 2-layer bounded-norm ReLU network that implements $\fstar$ and $\fthetabasemin = \min_{f\in\sH}\{\|f\|: \Pi\circ f(x) = \fstar(x),~x\in \sX\}$.
Then 
\begin{align}
    \frac{\|\fstd\|}{\|\fthetabasemin\|} = \Omega\left(\frac{NL}{U(|J| +\delta\frac{|I|}{\Delta_x})}\right).
\end{align}
\end{lemma}
\begin{proof}
    Using Lemma~\eqref{thm:normlowerbound}, we have $\|\fstd\|\geq \frac{NL}{\delta}$. Taking the (inverse) ratio with the upper bound in Lemma~\eqref{thm:upper-bound} (considering the second term of the maximum as a constant) implies that the (inverse) ratio between norms is
\begin{align*}
    \frac{\|\fthetabasemin\|}{\|\fstd\|} &< \frac{U}{NL}\left(|J| + \delta\frac{|I|}{\Delta_x}\right).
\end{align*}
Taking the inverse gives the result.
\end{proof}

\subsection{Multivariate outputs}

We extend Lemma~\ref{thm:univariate-app} to the multivariate output case, considering functions from $\R \rightarrow \R^k$.
Here, the denoiser is defined as $\Pi(y) \in \argmin_{v\in\sV} \|v - y\|_2$.
Similar to the univariate case, we will construct a lower bound on the complexity of $\fstd$ and upper bound on $\fthetabasemin$.
For the lower bound, extending to the multivariate case is straightforward: to fit the multivariate output, the model must also fit each individual output coordinate, and thus the maximum over the individual multivariate lower bounds will hold.
For the upper bound, we consider a 2-layer ReLU neural network family where the $j$-th output of $\ftheta(x)$ is $\fthetai(x)$, where each $\fthetai$ is a univariate 2-layer ReLU network, and measure its complexity. We upper bound the norm of $\fthetabasemin$ by the norm of the constructed $\ftheta$.
In our special case, we show the coordinate-wise and multivariate projections are the same, allowing us to leverage the univariate bound.

For each univariate network, the first layer weights are shared but each output $j$ has a different second layer weight.
We denote the second layer weights for the $j$-th network as $w^{(2, j)}\in \R^h$.
The norm of the 2-layer ReLU network $\ftheta$ with $k$-dimensional outputs is defined as
\begin{align}
    C(\theta) = \frac{1}{2}(\|W^{(1)}\|^2_F + \sum_{j=1}^k\|w^{(2, j)}\|^2_2).
\end{align}
Again, the norm associated with $f\in\sH$ is the minimum norm required to implement a given $f$:
\begin{align}
    \|f\| = \inf_{\theta \in \Theta} C(\theta) \text{~s.t.~} f_{\theta} = f.
\end{align}

Here, we define bounds for terms analogous to those in the univariate case.
We let $J_j$ and $I_j$ be the index set of non-adjacent and adjacent pairs of intervals respectively in the $j$-th output coordinate.
Let $y^*_{i, j}$ be the $j$-th output coordinate of the $i$-th valid point.
For the $j$-th output coordinate, let $L_j = \min_i |y^*_{i,j}-y^*_{i+1,j}|$ and $U_j = \max_i |y^*_{i,j} - y^*_{i+1, j}|$ be the min and max separation between valid points.
Let $\Delta_x$ be the length of the smallest interval in $\sX$.
Given these definitions, we show a similar gap in the multivariate output case.
\setcounter{theorem}{0}
\begin{theorem}
\label{thm:multivariate-discrete}
Let the valid output space $\sV$ be a set of $N$ discrete points $y^*_1,\dots,y^*_N$ in $\sV = \R^k$.
Let $\fstar$ be a piecewise constant function defined on $\sX = \uplus_{i=1}^n [x^l_i, x^u_i]$ with values in the discrete set $\sV$, and let $\delta>0$ be the gap between any two intervals in $\sX$.
Let $\fstd:\R\rightarrow \R^k$ be a multivariate bounded-norm 2-layer ReLU network in $\sH$ that implements $\fstar$ and $\fthetabasemin = \min_{f\in\sH}\{\|f\|: \Pi\circ f(x) = \fstar(x),~x\in \sX\}$.
Then 
\begin{align}
    \frac{\|\fstd\|}{\|\fthetabasemin\|} = \Omega\left( \frac{N\max_jL_j}{\sum_{j=1}^{k}U_j\left(|J_j|+\delta\frac{|I_j|}{\Delta_x}\right)}\right).
\end{align}
\end{theorem}
\begin{proof}
For the lower bound on $\|\fstd\|$, note that $\fstd$ must fit every coordinate of the output (a univariate regression problem). Thus, we can lower bound by the norms on any output coordinate, $\|\fstd\| \geq \|\fstdi \|$ for any $j$, where $\fstdi$ is the univariate 2-layer ReLU network for the $j$-th output coordinate. In particular, we can bound by the maximum of the norms, $\|\fstd \| \geq \max_j \|\fstdi \|$.

For the upper bound on $\| \fthetabasemin \|$, we construct a multivariate network $\ftheta$. 
For the $j$-th output coordinate $\fthetai$ of the construction, we ensure that $\Pi_j \circ \fthetai = \fstar_j$, where the projection for the $j$-th coordinate is $\Pi_j(y) = (\arg\min_{v\in \mathcal{V}}|v_j - y|)_j$. While generally the coordinate-wise and multivariate projections are not equal, they are equal here since the coordinate-wise projection is in $\mathcal{V}$: by definition, $\fstar_j(x)=(\arg\min_{v\in \mathcal{V}}|v_j - \fthetai(x)|)_j$ for all $j$, and thus for any valid $v \in \mathcal{V}$, $\sum_j(f^*_j(x) - \fthetai(x))^2 \leq \sum_j (v_j - \fthetai(x))^2$ so that the coordinate-wise projection and the multivariate projections are the same.
We assume for convenience that ties in the multivariate projection and the coordinate-wise projection are broken in the same way in these cases.
Thus, for the upper bound we focus on constructing $\ftheta$ where the coordinate-wise projection represents $\fstar$.

We take $f_{\alpha_j}$ to be an independent univariate 2-layer ReLU network that fits the $j$-th coordinate of $\fstar$ and has the same norm the corresponding univariate construction from Lemma~\ref{thm:univariate-app} that fits the $j$-th output coordinate.
Each $\alpha_j$ consists of $h_j$ hidden units with first layer weights $W^{(1)}_{\alpha_j}, b^{(1)}_{\alpha_j}, w^{(2)}_{\alpha_j}, b^{(2)}_{\alpha_j} \in \R^{h_j}$.
We construct $\fthetai$, the network that computes the $j$-th output of our construction $\ftheta$, by concatenating the first layer weights of each $f_{\alpha_j}$ to define the shared first layer weights
\begin{align*}
    W^{(1)} = [W^{(1)}_{\alpha_1}; \dots;W^{(1)}_{\alpha_k}] \in \R^h\\
    b^{(1)} = [b^{(1)}_{\alpha_1}; \dots; b^{(1)}_{\alpha_k}] \in \R^h
\end{align*}
where $h = \sum_{j=1}^k h_j$.
The second layer weights extend the corresponding second layer weights $w^{(2)}_{\alpha_j}$ with zeros for the newly introduced hidden units:
\begin{align*}
    w^{(2, j)} = [\mathbf{0}^{h_{j-}}; w^{(2)}_{\alpha_j}; \mathbf{0}^{h_{j+}}]\\
    b^{(2, j)} = [\mathbf{0}^{h_{j-}}; b^{(2)}_{\alpha_j}; \mathbf{0}^{h_{j+}}]
\end{align*}
where $h_{j-} = \sum_{r=1}^{j-1} h_r$ and $h_{j+} = \sum_{r=j+1}^k h_r$.
We define the $j$-th output of our construction $\ftheta(x)$ to be $\fthetai(x)$.
The norm of this concatenated network is $\|\ftheta \| = \sum_{j=1}^k \|\fthetai\|$. We bound $\|\fthetabasemin\| \leq \|\ftheta \|$.

Then using Lemma~\ref{thm:normlowerbound} and Lemma~\ref{thm:upper-bound} on each output coordinate, we have $\| \fstd \| \geq \frac{NL_j}{\delta}$ and $\|\fthetabasemin\| \leq \|\ftheta\| \leq \sum_{j=1}^k (\frac{|J_j|U_j}{\delta} + \frac{|I_j|U_j}{\Delta_x})$.
Taking the ratio gives the result.
\end{proof}

\section{Font image generation}\label{app:fonts}
We use a dataset of 56k fonts scraped from the internet~\cite{bernhardsson2016fonts}.
Out of the 6200 labeled examples (62 characters $\times$ 100 fonts), we split randomly into 2500 training examples, 100 validation examples, and 3600 test examples.
The training set contains 25 examples on average from each font, and the model must generate new characters from those fonts at test time.
We have additional unlabeled images for $\sim$50k other fonts to train the denoiser.
After learning the denoiser, we train the base model composed with the denoiser and minimize squared error. We set $\lambda=0$ in \eqref{eqn:composed-objective}, using only the composed component of the objective.
We tune the L2 regularization strength out of \{0, 0.1, 1e-2, 1e-3, 1e-4, 1e-5, 1e-6\} according to squared error on the validation set.
The denoiser is trained for two epochs on unlabeled data and all other models (baseline and composed) are trained for 1000 epochs on the labeled data, using the validation set for early stopping.

The baseline and base predictors use 7-layer feedforward networks where the one-hot vector for character identity is first fed through an embedding layer with embedding dimension 62, and the one-hot vector for the font uses an embedding layer of embedding dimension 100. These embeddings are then concatenated and fed through a 7-layer feedforward network.

Our models were trained on 1 Titan Xp or Titan RTX GPU. All image generation experiments took around 3 hours to run for each of the pre-training and fine-tuning steps.

\begin{figure}
\centering
\begin{tabular} {l r r }
\toprule
   & Test MSE\\
\midrule
    Baseline & 0.193 \\
    Composed (Emboss) & \textbf{0.187}\\
    Composed (Contrast) & \textbf{0.172}\\
    Composed (Gaussian blur) & \textbf{0.171}\\
\bottomrule
\end{tabular}
  \caption{Test MSE on font image generation for a variety of noising functions.}\label{fig:fonts-noising-fns}
\end{figure}

\begin{figure}
    \centering
\includegraphics[width=0.4\textwidth]{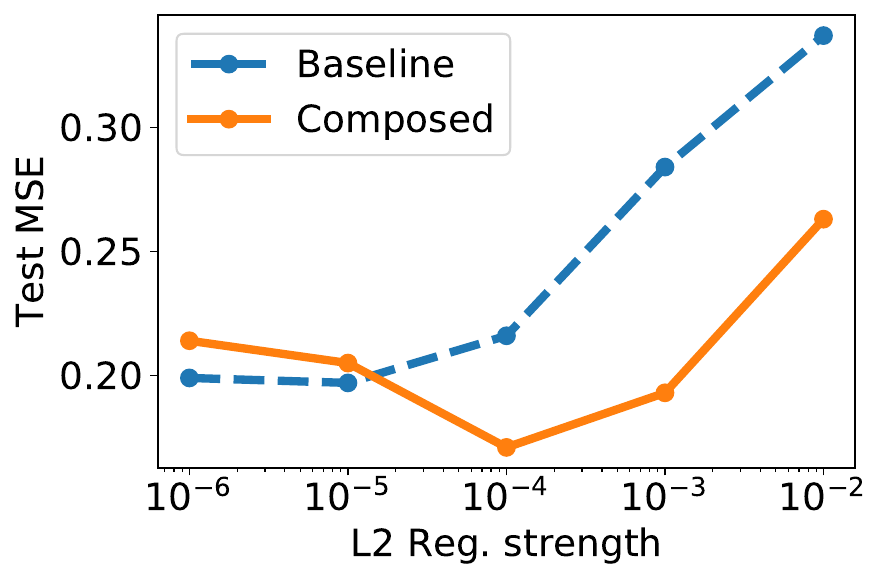}
  \caption{
    Varying L2 regularization strength (1e-6 to 1e-2) for direct and composed predictors. Increasing the regularization strength improves the composed predictor while hurting the direct predictor.
}\label{fig:fonts-numbers}
\vspace{-11px}
\end{figure}

\paragraph{Robustness to choice of denoising perturbations.} In Table~\ref{fig:fonts-noising-fns}, we test the robustness of composing with a denoiser on different choices of noising functions. In addition to Gaussian blur, we change the contrast of the image by a factor of 0.5 or emboss the image.
All noising functions improve beyond the baseline predictor.

\paragraph{Performance with varying labeled and unlabeled data size.} We trained the denoiser with 100 and 5000 additional unlabeled fonts (originally 50k unlabeled fonts). The Composed test MSE was 0.182 and 0.175 for 100 and 5000 unlabeled fonts respectively, well below the Baseline test MSE (0.193) and approaching the result with 50k unlabeled fonts (0.171). Varying labeled data size in font generation (500, 1500, 2500, 4500 examples), the Composed model has lower test MSE than Baseline on all sample sizes, with an average relative MSE decrease of 11\%.
We visualize these results in Figure~\ref{fig:fonts-labeled-unlabeled}.

\begin{figure}[tbp]
  \centering
  \subfloat[Vary labeled data]{\includegraphics[width=0.4\textwidth]{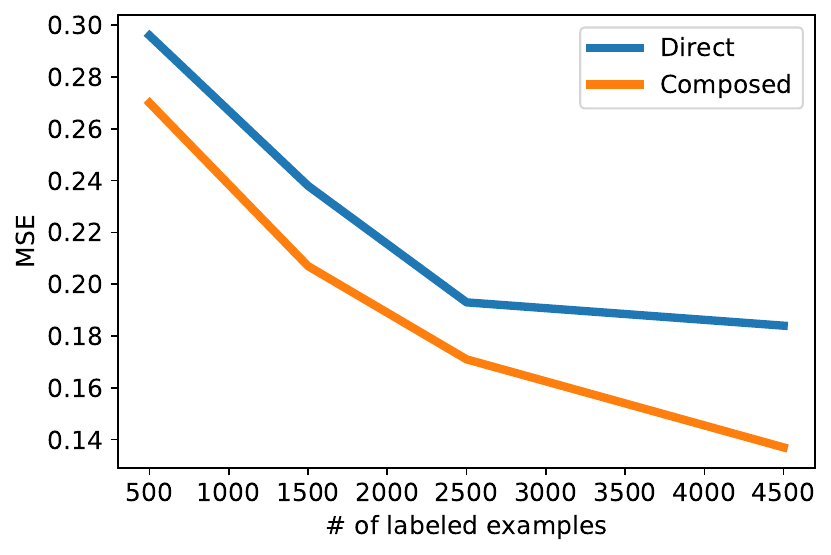}}
  \hfill
  \subfloat[Vary unlabeled data]{\includegraphics[width=0.4\textwidth]{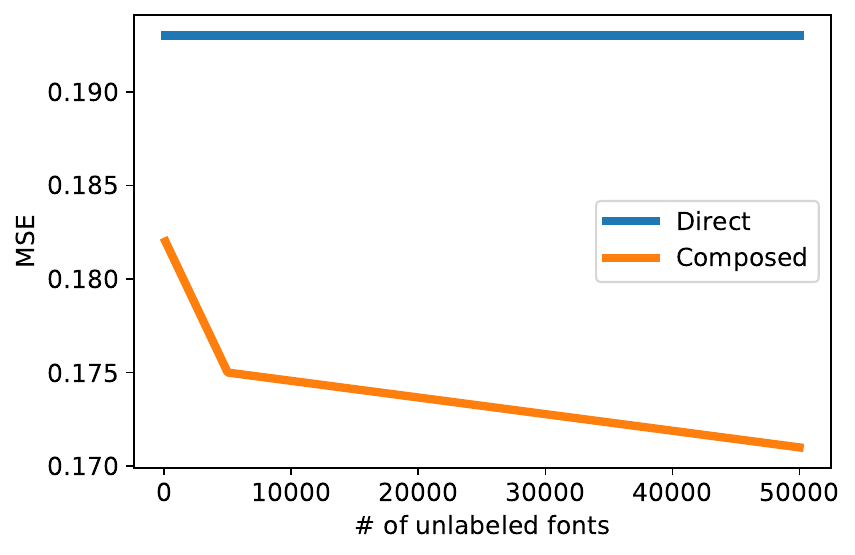}}
  \caption{Performance with varying labeled and unlabeled data (while fixing the other to the default in the original experiment) in font generation. The gap between the direct and composed models increases with increasing labeled and unlabeled examples. }\label{fig:fonts-labeled-unlabeled}
\end{figure}

\section{\sanstype~pseudocode-to-code dataset}
\label{app:synthetic}
Each program in the dataset is generated in two phases. Each line in the program is templated and have pseudocode and noisy code templates to accompany it. When there are multiple options for any of code, pseudocode, or noisy code, a uniformly random option is sampled.

In the first phase, 1 to 4 variables with random types (\texttt{bool, string, int}) are instantiated.
They may be assigned a random value or take on a value from \texttt{stdin}.
Integers take random values in 0 to 100 and there are 10 possible strings in the form \texttt{str\_i} for $i\in \{0,\dots,9\}$.
There are 10 possible variable names in the form \texttt{var\_i} for $i\in \{0,\dots,9\}$.

In the second phase, up to 5 additional code lines are generated. Twenty percent of the time, a new variable with a random type is instantiated and assigned a value. The other 80\% consists of type-preserving operations on an existing variable or variables, such as setting a value, prepending/concatenating, adding/subtracting an integer, taking the logical AND, and conditional swap based on the value of a boolean variable.

The dataset presents the following challenges for the predictor, both stemming from ambiguity in the pseudocode:
\begin{enumerate}
    \item Type inference: the words `set' and `add' are used almost everywhere for instantiation, initialization, and type-preserving operations.
    \item Initialization: A particular example of ambiguity is that the pseudocode for value updates and initialization are often the same (\texttt{set <var> to <value>}), thus requiring the predictor to look in the pseudocode context to figure out whether to insert a type and what the correct type to insert should be.
    \item Scopes and variable shadowing: Initialization can be particularly tricky during a conditional swap, where a scope is introduced with the if block. Inside the if block, one may reinitialize a variable that already exists (shadowing) or initialize a new variable which does not already exist. If the variable is not shadowed, it may change the value of the variable in the outer scope, causing a change in the program output. A locally-scoped variable cannot be used outside the if-block, so they may have to be reinitialized. In this dataset, variables are always newly initialized in the if-block.
\end{enumerate}

\paragraph{Perturbations for denoising.}
We generate perturbations for denoising by inducing random semantic and syntactic changes:
\begin{enumerate}
    \item Replacing, deleting, or inserting a random type in instantiation, initialization, and variable modification statements.
    \item Print and input statements: removing double arrows (\texttt{<<}), reversing arrow directions, removing \texttt{cout}.
\end{enumerate}

\paragraph{Experimental details.}
We use the sentencepiece BPE tokenizer~\cite{kudo2018sentencepiece,sennrich2016bpe} with a joined dictionary size of 600. There are 1000 labeled pairs of pseudocode to code and 20000 unlabeled code examples. Separate validation and test sets of 500 examples are independently generated.
The baseline model, standard fine-tuning model, base predictor, and denoiser use a Transformer architecture with encoder/decoder using 3 (or 6) layers and 2 attention heads, embedding dimension 256, and FFN embedding dimension 1024.
We use dropout probability 0.1, attention dropout probability 0.2, ReLU dropout probability 0.2, weight decay 0.0001. We use the cross entropy loss.
We train the baseline and backtranslation models (models trained frome scratch) for 500 epochs and pre-train the denoiser for 50 epochs. We do standard fine-tuning and composed fine-tuning for 300 epochs.
For all predictors, we use the validation set to do early stopping.
For composed fine-tuning, we use the cross entropy loss of the composition $\Pi \circ \fthetabase$ on the validation set for early stopping.

\section{Training details}\label{app:objective}
If we exactly compute the log-likelihood of the denoised output instead of the first term of the composed objective (Equation~\eqref{eqn:composed-objective}), this would require computing $\int_{y'} p_\beta(y \mid y')p_\theta(y' \mid x)$, marginalizing over outputs $y'$.
Since the output space is large, this is intractable.
Instead, we optimize a lower bound on the log-likelihood of the composed predictor.
To see this, note that the log-likelihood of the composed predictor is lower bounded by the first term of the objective via Jensen's inequality:
\begin{align}
    \E_{x,y}[ \log~\E_{y’ \sim p_\theta(\cdot \mid x)}[p_\beta (y \mid y’)]] \geq \E_{x,y}[ \E_{y’ \sim p_\theta(\cdot \mid x)}[\log p_\beta (y \mid y’)]].
\end{align}

To optimize Equation~\eqref{eqn:composed-objective} where the output space is discrete, we use the REINFORCE estimate of the gradient for the first term~\cite{williams1992simple}. We do not use the Gumbel-softmax reparameterization here since the Transformer model autoregressively conditions on its discrete outputs, but different choices of the base model can enable use of reparameterized gradients.

\label{app:reinforce}
The composed predictor optimizes
\begin{align}
    \argmax_\theta \E[\E_{y' \sim p_\theta}[\log p_\beta(y \mid y') ]] + \lambda \E[\log p_\theta(y \mid x)]
\end{align}
where the outer expectations are over the data, where the samples are $(x, y)$ pairs.
Here, $\theta$ are the parameters of the learned model and $\beta$ are the parameters of the denoiser.
The score function estimate of the gradient is
\begin{align}
    \nabla_\theta \E[\E_{y' \sim p_\theta}[ \log p_\beta(y\mid y') \nabla_\theta \log p_\theta(y' \mid x)]] + \lambda \E[\nabla_\theta\log p_\theta(y \mid x)].
\end{align}

For examples that the model predicts wrongly, the model is encouraged to put smaller probabilities on these examples. This may cause the first term of the gradient to have a very large magnitude, causing numerical instability. We resolve this by clamping $\log p_\beta(y\mid y')$ to $\max(\log p_\beta(y \mid y'), \gamma)$ where $\gamma=-50$ in all experiments. Even after clamping, the composed loss is roughly an order of magnitude larger than the standard loss; to normalize the losses, we scale the composed loss by 0.1.
In all pseudocode-to-code experiments, we use $\lambda=1$ to balance between the fitting the objective directly and using the composed objective.

Our models were trained on 1 Titan Xp or Titan RTX GPU. All \sanstype~experiments took around 3 hours to run, which includes pre-training and fine-tuning steps.

\paragraph{BART-style fine-tuning.}
We implement a two-stage fine-tuning method similar to BART~\cite{lewis2020bart} as follows. First, we fix the Transformer decoder and train only the encoder for 200 epochs. Then, we fine-tune the encoder-decoder together for 100 more epochs. We note that the BART paper adds extra layers to the encoder for fine-tuning, while we fine-tune the existing encoder. The first step has similarities to our method, and our theory may give some justification for this step as well.

\section{\spoc}
\label{app:spoc}
\paragraph{Denoising objective.}
We use 284477 unlabeled code examples from codeforce.com to generate 1137908 pairs of noisy code to valid code. For each unlabeled program, we generate 1 unnoised example and 3 noised examples, where each noised example has one line with an error. We follow~\cite{yasunaga2020repair} to generate error lines by random semantic and syntactic changes, including insertion, deletion, and replacement of keywords, variables, and syntactical symbols.

\paragraph{Pre-trained model.}
We use random syntactic and semantic corruptions of additional $\sim$280k unlabeled code examples from \texttt{codeforces.com} as in~\citet{yasunaga2020repair}.
Previous program repair works~\cite{yasunaga2020repair} utilize compiler error messages to guide the repair model.
We only use code as input, and thus the task is relatively difficult.
We define $p_\Pi$ in two parts. First, we train a binary classifier $g_\gamma:\sY \rightarrow \{0,1\}$ which detects if a program has an error (error is label 1), trained using the denoising dataset. For an output $y'$, if $g_\gamma(y')=0$ then we define $p_\Pi(y\mid y')=\delta(y')$ to be a delta distribution on $y'$.
Otherwise, if $g_\gamma(y')=1$, then $p_\Pi(y\mid y') = p_\nu(y\mid y')$, where $p_\nu$ is a Transformer. The Transformer $p_\nu$ is first pretrained using a linewise code repair dataset generated from unlabeled examples, then trained on full-program repair where the input program has one random corrupted line with probability 0.75.
Thus, taking the parameters of $\Pi$ to be $(\gamma, \nu)$, we have $\Pi(y')=y'$ if $g_\gamma(y')=0$ and $\Pi(y') = \argmax_y p_\nu(y \mid y')$ otherwise.

\paragraph{Data processing.} Since the \spoc~dataset contains a small fraction of programs which have a large number of tokens, we filter out the longest examples from the training data. After filtering, we retain over 95\% of the original training set.
Similarly to \sanstype, special symbols (\$ and $\sim$) are added to delineate lines and tabs in the pseudocode and code, and we preprocess the code using a byte-pair encoding using SentencePiece~\cite{kudo2018sentencepiece}, with joined vocabulary size 10000.

\paragraph{Data filtering.} We train on a truncated version of the \spoc~dataset~\cite{kulal2019spoc}.
We filter out an example during preprocessing if, after adding special tokens for tab and code lines, the number of characters exceeds 1000. This retains 11355 examples out of the full 11925 training examples. We use the given validation splits.
The test sets are unfiltered.

\paragraph{Training details.} For \spoc~experiments, we use a Transformer architecture for all models with 5 layers, 8 attention heads, embedding dimension 256, and FFN embedding dimension 1024. We use this architecture for both the denoiser and the models.
We use dropout probability 0.4, attention dropout probability 0.2, ReLU dropout probability 0.2, weight decay 0.0001, taken as reasonable defaults from~\citet{guzman2019flores}.
We use a decaying label smoothing schedule with smoothing parameter starting with 0.2 for 100 epochs, then 0.1 and 0.05 for 25 epochs each.
We found that reducing the label smoothing parameter near the end of training improves generalization for all models.
The composed predictor initializes from standard fine-tuned model and is trained for 20 epochs, taking the best model according to (label-smoothed) cross entropy loss of the composition $\Pi \circ \fthetabase$ on the validation set.
For the denoiser, we also use a decaying label-smoothing schedule with smoothing parameter starting at 0.2 for 5 epochs, then decaying to 0.1 for 1 epoch, and finally 0.0 for 3 epochs.
For backtranslation models, we first train a code-to-pseudocode model using the labeled data and use this model to produce synthetic pseudocode examples for unlabeled code. Then, we train a pseudocode-to-code model using the labeled examples and synthetically generated examples. Finally, we use this model as initialization to finetune on the labeled data only.

Our models were trained on 1 Titan Xp or Titan RTX GPU. All \spoc~experiments took around 10 hours for the pre-training and fine-tuning steps.

\end{document}